\documentclass[preprint,12pt,authoryear]{elsarticle}

\usepackage{amssymb}
\usepackage{amsthm}
\usepackage[fleqn]{amsmath}
\usepackage{mathtools}
\usepackage{mathrsfs}
\usepackage{graphicx}
\usepackage{subcaption}
\usepackage[space]{grffile}
\usepackage{url}
\usepackage{booktabs}
\usepackage{multirow}
\usepackage{array}
\usepackage{enumitem}
\usepackage[ruled, vlined, linesnumbered]{algorithm2e}
\usepackage{placeins}
\usepackage{float}
\newtheorem{proposition}{Proposition}

\usepackage[pagewise, mathlines]{lineno}

\usepackage[colorlinks]{hyperref}
\hypersetup{
  citecolor = {blue}
}

\usepackage{geometry}
\begin{document}

\begin{frontmatter}

\title{Hierarchical Reinforcement Learning for Neural Network Compression (HiReLC): Pruning and Quantization}

\author[ensia]{Kamar Hibatallah Baghdadi\corref{cor1}}
\ead{kamar.baghdadi@ensia.edu.dz}

\author[ensia]{Kawther Guoual Belhamidi}

\author[ensia]{Sara Belhadj}

\author[ensia]{Aissa Boulmerka}

\author[uge]{Nadir Farhi}

\cortext[cor1]{Corresponding author.}

\address[ensia]{National School of Artificial Intelligence (ENSIA), Sidi Abdellah Campus, Algiers, Algeria}

\address[uge]{Cosys-Grettia, Univ Gustave Eiffel, F-77454 Marne-la-Vallee, France.}

\begin{abstract}
We present \textbf{HiReLC}, a hierarchical ensemble-reinforcement learning framework
for automated joint quantization and structured pruning of deep neural
networks. The framework decomposes the compression search across two levels of
abstraction: low-level agents~(LLAs) operate independently per block, selecting
per-kernel configurations over a multi-discrete action space spanning bitwidth, pruning
keep-ratio, quantization type, and granularity, while high-level agents~(HLAs)
coordinate global budget allocation via ensemble voting guided by Fisher
Information--based sensitivity estimates. To mitigate the computational cost of policy
evaluation, an iterative active learning loop interleaves surrogate-guided RL
optimization with post-compression fine-tuning, using a lightweight MLP surrogate to
amortize expensive evaluations and a logit-MSE proxy during cold-start. The surrogate is
used for reward shaping rather than as a replacement for final post-compression
evaluation. The controller is architecture-agnostic by design, with a modular layer
abstraction decoupling the RL environment from the underlying network topology. Experiments across Vision Transformer
and CNN benchmarks demonstrate effective parameter-storage compression ratios of
5.99--6.72$\times$ with a 3.83\% gain in one setting and 0.55--5.62\% accuracy drops
elsewhere, supporting hierarchical policy decomposition and sensitivity-aware guidance
as practical design choices for joint neural network compression.
\end{abstract}


\end{frontmatter}


\section{Introduction}
Deep neural networks have achieved remarkable performance across visual recognition 
tasks, yet their deployment on resource-constrained platforms remains challenged by 
high parameter counts, memory footprints, and inference latencies. Vision 
Transformers~(ViTs) exemplify this challenge in modern vision systems 
\citep{dosovitskiy2021,islam2022visiontransformer}, and the efficiency gap is pervasive 
across architectural families, motivating the need for general-purpose compression 
frameworks.

Existing compression approaches, whether based on static heuristics, one-shot 
mixed-precision assignment, or single-objective optimization, fail to capture 
cross-layer dependencies and lack adaptive refinement during fine-tuning. Model 
compression techniques spanning pruning and quantization have been explored broadly 
\citep{he2018amc,wang2019haq,bondarenko-2021-transformer-quantization}, but reinforcement 
learning methods still face high variance in large action spaces, coarse granularity, 
and limited sensitivity-aware guidance.

We note a related but orthogonal line of work that compresses reinforcement-learning
agents themselves during training~\citep{graesser2022state,obandoceron2024pruned,vincent2025eaudeq};
HiReLC instead compresses pre-trained target networks as a post-training procedure.

To address these limitations, we propose \textbf{HiReLC}, an architecture-agnostic 
hierarchical ensemble-reinforcement learning framework for joint 
quantization and structured pruning. Low-level agents~(LLAs) operate per block over 
a multi-discrete action space, while high-level agents~(HLAs) coordinate global budget 
allocation via Fisher Information--based sensitivity estimates. An iterative active 
learning loop interleaves surrogate-guided RL optimization with post-compression
fine-tuning across an ensemble of heterogeneous RL policies.

The main contributions of this work are summarized as follows:
\begin{itemize}
    \item We propose \textbf{HiReLC}, an architecture-agnostic hierarchical ensemble-RL
 framework for \emph{joint} structured pruning and mixed-precision quantization.
 The framework decomposes the search across two tiers: a High-Level Agent
 (\textbf{HLA}) ensemble responsible for global budget allocation, and per-block
Low-Level Agent (\textbf{LLA}) ensembles operating over a four-dimensional action
space (bitwidth, keep-ratio, quantization type, granularity). Unlike single-level
controllers targeting a single compression axis~\citep{he2018amc,wang2019haq},
HiReLC achieves 5.99--6.72$\times$ effective parameter-storage compression across
four architectures, with 0.55--5.62\% drops in loss cases and a 3.83\% gain in one
setting.
  \item We integrate Fisher Information sensitivity scores~\citep{theis2018faster} into
both hierarchy levels: initializing LLA observations, contributing to the sensitivity
penalty $P_{\mathrm{sens}}$, and driving HLA budget corrections,
ensuring critical blocks receive lighter compression without additional search-time evaluations. \\
Prior RL methods rely on magnitude-based heuristics. Our sensitivity-aware design
propagates gradient-based block importance through both reward and action-correction,
improving trade-offs across ViTs~\citep{dosovitskiy2021} and
CNNs~\citep{he2016deep,sandler2018mobilenetv2}.
  \item We introduce a surrogate-augmented active learning loop interleaving RL optimization
with post-compression fine-tuning. A lightweight MLP surrogate amortizes expensive
evaluations once a minimum number of configurations have been evaluated
(see Section~\ref{sec:surrogate}); a logit-MSE proxy covers the cold-start phase,
while reported surrogate MAEs of 3.57--15.20\% quantify when the learned proxy is
reliable and when it remains noisy. \\
Unlike Hessian-based methods~\citep{dong2020hawqv2} that incur high computation costs,
our surrogate loop reduces the number of full model evaluations required during search,
enabling a practical framework-level study within a small number of cycles.
\end{itemize}


\section{Related Work}

Vision Transformers have achieved state-of-the-art performance across vision tasks, but their computational and memory demands pose significant deployment challenges \citep{dosovitskiy2021,islam2022visiontransformer}. Early efficiency efforts focused on token reduction or dynamic computation skipping \citep{rao2021dynamicvit,ryoo2021tokenlearner}, though these do not address parameter-level compression. Compressing ViTs via pruning or quantization is particularly challenging due to the sensitivity of self-attention mechanisms \citep{lu2019transformer-dynamics}, motivating adaptive, structure-aware strategies.

Post-training quantization often degrades Transformer performance at low bit-widths \citep{bondarenko-2021-transformer-quantization}, while quantization-aware training methods such as LSQ \citep{esser2020learnedstepsizequantization} better adapt to low-precision constraints. Structured pruning (removing attention heads \citep{michel2019heads}, feed-forward neurons \citep{wang2020structuredpruning}, or entire blocks) has become the dominant paradigm, as unstructured sparsity yields limited real-world speedups \citep{gale2019sparsity}.

Reinforcement learning has been explored for automating compression decisions. AMC \citep{he2018amc} and HAQ \citep{wang2019haq} formulate compression as sequential decision-making, using a DDPG agent to select layer-wise pruning ratios or mixed-precision assignments, demonstrating effectiveness on CNNs such as MobileNet and ResNet. However, both employ single-level controllers on a single compression axis. HAWQ-V2 \citep{dong2020hawqv2} uses Hessian trace information to guide mixed-precision quantization, achieving strong compression ratios at minimal accuracy cost, though at significant computational expense. I-ViT \citep{li2023ivit} demonstrates robust integer-only quantization for ViTs via QAT but does not address structural pruning. DeepCompress-ViT \citep{ahmed2025deepcompress} is orthogonal to structured pruning: it combines an encoder--decoder weight representation with low-bit quantization and unified compression training to target ViT storage and edge efficiency, rather than learning a structured channel-selection policy.

In contrast, our work introduces a hierarchical ensemble-RL framework that jointly optimizes quantization and structured pruning, incorporating sensitivity-aware priors, ensemble-based agents, and an active learning compression loop to overcome the limitations of static policies and unstable training under aggressive compression targets. The term ``ensemble'' is used deliberately: HiReLC trains independent PPO/A2C policies and combines their proposals by voting, rather than modeling cooperative or competitive multi-agent interaction with communication, centralized critics, or value factorization as in MARL methods such as MADDPG and QMIX~\citep{lowe2017multiagent,rashid2018qmix}.

\begin{table}[htbp]
    \caption{Notation summary for recurring symbols and table abbreviations. Retained-size fractions are smaller when compression is stronger; compression ratios are larger when compression is stronger.}
    \label{tab:notation}
    \begin{center}
        \begin{tabular}{p{0.25\linewidth}p{0.75\linewidth}}
            \toprule
            \textbf{Symbol / term} & \textbf{Meaning} \\
            \midrule
            $\mathcal{M}$, $\boldsymbol{\theta}$, $\boldsymbol{\theta}_0$, $\boldsymbol{\theta}'$, $\boldsymbol{\theta}^*$ & Model, original, fine-tuned baseline, compressed, and best parameters  \\
            $\mathbf{c}$, $c_{i,k}$ & Full and per-kernel compression configuration  \\
            $\rho_{i,k}$ & Keep-ratio: fraction of output channels retained  \\
            $\tau_{i,k}$ & Quantization type \\ 
            $\mu_{i,k}$ & Granularity \\
            $\nu_{i,k}$ & Kernel retained-size fraction $(b_{i,k}/b_0)\rho_{i,k}$ \\
            $\bar{\nu}_i$ & Block retained-size fractions \\
            $\bar{\nu}$ & Global retained-size fractions \\ 
            $\bar{\nu}^*$ & Equivalent retained-size target \\
            $CR$ & $1/\bar{\nu}$ \\
            $\mathcal{R}(\mathbf{c})$ & Effective parameter-storage compression ratio \\
            $R_{\mathrm{target}}$ & Target CR lower bound  \\
            $\ell_i$  & HLA compression tier \\
            $\sigma_i$ & decoded strategy \\
            $p_i$ & scalar strategy code \\
            $\mathcal{B}_i$ & HLA LayerBudget assigned to block $B_i$ \\
            $\Delta_{\mathrm{acc}}$, $\Delta\mathcal{A}^{(t)}$ & Allowed and measured accuracy drops \\
            $r_t^{\mathrm{LLA}}$, $r_t^{\mathrm{HLA}}$ & Tier-specific RL rewards \\
            $\mathcal{P}$, $\mathcal{S}$, $\mathcal{U}$, $\mathcal{T}$, $\mathcal{W}$ & MDP tuple, state/action spaces, transition, reward \\
            $\mathbf{s}$, $\mathbf{s}_t$, $\mathbf{a}_t$, $\mathbf{a}_h$ & HLA state, LLA state, LLA action, and HLA action \\
            $n_a$, $n_h$ & LLA/HLA ensemble sizes \\
            $w_j$ & LLA vote weight \\
            $\mathcal{X}_c$ & Calibration mini-batches \\ 
            $\mathbf{z}_{\mathrm{base}}$ & Cached baseline logits \\            
            $B_i$, $N$, $k$, $K_i$, $K_{\mathrm{tot}}$ & Block/index count, kernel/index count, and total kernels \\
            $b_{i,k}$, $b_{\min}$, $b_{\max}$, $b_0$ & Kernel bitwidth, bitwidth bounds, FP32 reference ($b_0=32$) \\
            $\rho_{\min}^{\mathrm{prune}}$, $\rho_{\max}^{\mathrm{prune}}$ & Minimum/maximum channel-removal fractions \\
            $P_{i,k}$, $P_i$, $P$ & Kernel, block, and global FP32 parameter counts \\
            MSR & Model-size reduction, $1-\bar{\nu}=1-1/\mathrm{CR}$ \\
            $R_i$, $\mathbf{R}$, $\bar{R}$ & HLA block target retained-size fraction, vector, and mean \\
            $\rho_{\max,i}$, $b_{\min,i}$, $\Delta_{\mathrm{acc},i}$ & Block-level pruning, bitwidth, and accuracy-drop limits \\
            $\mathcal{A}(\cdot)$, $\mathcal{A}_{\mathrm{base}}$, $\mathcal{A}^{(t)}$ & Validation, baseline, and cycle-$t$ accuracies \\
            $s_i$ & Fisher sensitivity score of block $B_i$ \\
            $\alpha,\beta,\gamma$; $\alpha_j,\beta_j,\gamma_j$ & Objective weights; agent-specific reward weights \\
            $\gamma_{\mathrm{RL}}$, $\rho_0$ & RL discount factor and initial state distribution \\
            $\delta_{b,i,k}$, $\delta_{\rho,i,k}$, $\delta_{\tau,i,k}$, $\delta_{\mu,i,k}$ & Discrete LLA action indices \\
            $\mathcal{B}_{\mathrm{rep}}$, $N_{\mathrm{rep}}^{\min}$, $\hat{f}$ & Surrogate replay buffer, activation threshold, MLP \\
            \bottomrule
        \end{tabular}
    \end{center}
\end{table}


\section{Methodology}

This section presents the proposed Hierarchical Reinforcement Learning (HiReLC)
framework for automated neural network compression. The framework is
architecture-agnostic at the controller level: the hierarchical agent topology, reward
formulations, ensemble voting mechanism, configuration-controlled quantization type and
granularity scheme, and active learning loop are invariant across architectures. Only
the layer abstraction module, sensitivity estimation pass, and downstream evaluation
loader are architecture- or dataset-specific.

Figure~\ref{fig:architecture_overview} summarizes the end-to-end pipeline. A
pre-trained model is parsed into compressible blocks, Fisher sensitivity is computed on a
small calibration subset, the HLA ensemble assigns per-block budgets, LLA ensembles refine
those budgets into per-kernel compression configurations, and the resulting compressed
model is fine-tuned before the replay buffer and surrogate are updated.

\subsection{Problem Formulation}

\subsubsection{Compression Function Definition}
\label{sec:compression_function}

Let $\mathcal{M}$ denote a pre-trained neural network with parameter vector
$\boldsymbol{\theta} \in \mathbb{R}^P$, consisting of $N$ compressible blocks
$\{B_i\}_{i=1}^N$, where block $B_i$ contains $K_i$ parametric kernels and
$K_{\mathrm{tot}}=\sum_{i=1}^{N}K_i$. Each block $B_i$ groups
logically associated \texttt{Linear}/\texttt{Conv2d} modules (e.g., QKV, attention
output, MLP projections), described by a \textit{KernelConfig} and aggregated into a
\textit{LayerConfig} that exposes per-kernel fields and summary statistics to higher-level
agents. We denote by $\boldsymbol{\theta}_0$ the uncompressed baseline after any
pre-compression fine-tuning or calibration stage; when no such stage is used,
$\boldsymbol{\theta}_0=\boldsymbol{\theta}$. A compression configuration is the joint assignment
$\mathbf{c} = \{c_{i,k}: i=1,\ldots,N;\; k=1,\ldots,K_i\}$, where each $c_{i,k}$ specifies the
compression decisions applied to kernel $k$ in block $B_i$. Formally, we define a
deterministic compression function:
\begin{equation}
    \mathcal{C}:\; \boldsymbol{\theta} \times \mathbf{c} \;\mapsto \; \boldsymbol{\theta}'
    \label{eq:compression_function}
\end{equation}
where $\boldsymbol{\theta}'$ is the parameter vector resulting from compression.\\
\noindent Each $c_{i,k}$ jointly specifies:
\begin{enumerate}[label=(\roman*)]
    \item the quantization bitwidth $b_{i,k} \in \{b_{\min}, \ldots, b_{\max}\}$;
    \item the structural pruning keep-ratio $\rho_{i,k} \in (0, 1]$, where $\rho_{i,k}$
          is the \emph{fraction of output channels retained};
    \item the quantization numeric type $\tau_{i,k} \in \{\mathrm{INT},\,\mathrm{FLOAT}\}$,
          selected per-kernel in the unconstrained regime, or fixed globally via a type
          override parameter (see Section~\ref{sec:lla}); and
    \item the granularity scheme
          $\mu_{i,k} \in \{\text{uniform},\,\text{log},\,\text{per-channel},\,\text{learned}\}$,
          selected per-kernel in the unconstrained regime, or fixed globally via a
          granularity override parameter (see Section~\ref{sec:lla}).
\end{enumerate}

\noindent The four granularity modes are operationally distinct: \textit{uniform} uses a
single per-tensor scale (max-abs), \textit{log} quantizes in the log domain with sign
restoration, \textit{per-channel} uses an independent scale per output channel, and
\textit{learned} applies per-channel \emph{sign-magnitude} quantization (1 sign bit and
$(b_{i,k}-1)$ magnitude bits), preserving near-zero weights at equal storage.

\noindent The optimization variable is the compression configuration $\mathbf{c}$, not
the deterministic compression operator $\mathcal{C}$. The goal is to identify the
optimal joint assignment $\mathbf{c}^*$ according to the multi-objective criterion
introduced next. The operator $\mathcal{C}$ produces the initial compressed weights;
the reported final parameters $\boldsymbol{\theta}^*$ are the best post-compression
fine-tuned weights initialized from $\mathcal{C}(\boldsymbol{\theta}_0,\mathbf{c}^*)$.

\begin{figure}[htbp]
    \centering
    \includegraphics[width=1\linewidth]{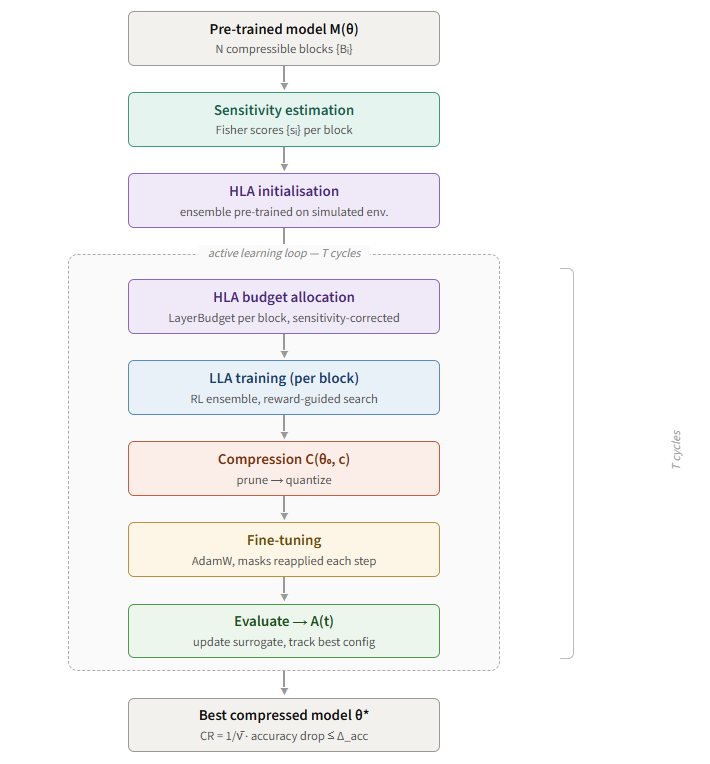}
    \caption{Overview of the HiReLC compression pipeline. Structural parsing and Fisher sensitivity estimation initialize the hierarchy; the HLA ensemble allocates layer budgets, LLA ensembles produce per-kernel configurations, and the active loop alternates compression, fine-tuning, and surrogate updates.}
    \label{fig:architecture_overview}
\end{figure}

\subsubsection{Multi-Objective Optimization}

\begin{equation}
    \mathbf{c}^* = \arg\max_{\mathbf{c}}\; \mathcal{F}(\mathbf{c}), \quad
    \mathcal{F}(\mathbf{c}) = \alpha \cdot
                              \mathcal{A}\!\left(\mathcal{C}(\boldsymbol{\theta}_0, \mathbf{c})\right)
                              + \beta  \cdot \mathcal{R}(\mathbf{c})
                              - \gamma \cdot \mathcal{V}(\mathbf{c})
    \label{eq:optimization_objective}
\end{equation}

\noindent Here $\mathcal{A}(\boldsymbol{\vartheta})$ denotes the top-1 validation
accuracy of the network instantiated with parameter vector $\boldsymbol{\vartheta}$.
Thus, $\mathcal{A}(\boldsymbol{\theta}_0)$ is the baseline accuracy before compression, while
$\mathcal{A}\!\left(\mathcal{C}(\boldsymbol{\theta}_0, \mathbf{c})\right)$ is the
accuracy after applying configuration $\mathbf{c}$.

\noindent subject to:

\begin{equation}
    \mathcal{A}(\boldsymbol{\theta}_0) -
    \mathcal{A}\!\left(\mathcal{C}(\boldsymbol{\theta}_0, \mathbf{c})\right)
    \leq \Delta_{\mathrm{acc}},
    \qquad
    \mathcal{R}(\mathbf{c}) \geq R_{\mathrm{target}}
    \label{eq:constraints}
\end{equation}

\noindent where $\mathcal{R}(\mathbf{c})$ is the global effective parameter-storage
compression ratio (reported as CR; Eq.~\eqref{eq:global_cr}); $\mathcal{V}(\mathbf{c})$
is a variance-based stability regularizer (std.\ of the LLA reward over the last five
steps); and $\{\alpha,\beta,\gamma\}$ are scalar trade-off weights.
$\Delta_{\mathrm{acc}} > 0$ is the allowed accuracy drop. $R_{\mathrm{target}} \geq 1$
is a target lower bound on the global compression ratio, not a retained-size fraction.
It is related to the global target retained-size fraction $\bar{\nu}^* \in (0,1)$ by
$R_{\mathrm{target}} = 1/\bar{\nu}^*$; e.g., $\bar{\nu}^* = 0.25$ implies
$R_{\mathrm{target}} = 4$, so the compressed model should retain at most $25\%$ of the
FP32 parameter storage. The configuration space scales as
$\mathcal{O}(|b|^{K_{\mathrm{tot}}} \cdot |\rho|^{K_{\mathrm{tot}}}
\cdot |\tau|^{K_{\mathrm{tot}}} \cdot |\mu|^{K_{\mathrm{tot}}})$
\cite{han2016deepcompression, he2018amc, wang2019haq},
motivating the hierarchical RL decomposition in Section~\ref{sec:hierarchical_agents}.

\subsubsection{Markov Decision Process Formulation}
\label{sec:mdp}

Each MDP instance is associated with one of two agent tiers: the \textbf{Low-Level
Agent} (\textbf{LLA}), governing the compression policy of an individual block $B_i$,
and the \textbf{High-Level Agent} (\textbf{HLA}), operating at the network level and
allocating compression budgets across all $N$ blocks.
\begin{equation}
    \mathcal{P} = \langle\, \mathcal{S},\; \mathcal{U},\; \mathcal{T},\;
                             \mathcal{W},\; \gamma_{\mathrm{RL}},\; \rho_0 \,\rangle
    \label{eq:mdp_tuple}
\end{equation}

\noindent Here, $\mathcal{S}$ is the state/observation space, $\mathcal{U}$ is the
action space, $\mathcal{T}$ is the transition rule induced by applying compression
actions and updating the environment state, $\mathcal{W}$ is the tier-specific reward
function, $\gamma_{\mathrm{RL}}$ is the RL discount factor, and $\rho_0$ is the initial
state distribution. Throughout this subsection, $R_i \in (0,1]$ denotes the
HLA-assigned target retained-size fraction for block $B_i$ (smaller values imply stronger
compression), not the global compression-ratio target $R_{\mathrm{target}}$. The
realized block retained-size fraction after an LLA action is $\bar{\nu}_i$, with
$\mathbf{R}=(R_1,\ldots,R_N)$, $P_i=\sum_{k=1}^{K_i}P_{i,k}$, and $P=\sum_i P_i$.
The global proposed target and realized retained-size fractions are parameter-weighted:
$\bar{R}=P^{-1}\sum_i P_iR_i$ and $\bar{\nu}=P^{-1}\sum_i P_i\bar{\nu}_i$.

The MDP formulation is preferable to a one-step bandit abstraction for three reasons.
First, HLA actions are globally coupled: the budget vector $\mathbf{R}$ must jointly
allocate per-block targets while keeping the mean target $\bar{R}$, and ultimately the
realized mean retained-size fraction $\bar{\nu}$, close to $\bar{\nu}^*$ rather than optimizing
each block independently. Second, the HLA state contains previous-cycle feedback
$(\Delta\mathcal{A}^{(t-1)},\bar{\nu}^{(t-1)})$, where
$\Delta\mathcal{A}^{(t-1)}$ is the measured accuracy drop in cycle $t-1$ and
$\Delta_{\mathrm{acc}}$ is the allowed accuracy-drop threshold. For the first cycle,
we set $\Delta\mathcal{A}^{(0)}=0$ and $\bar{\nu}^{(0)}=1$, corresponding to the
uncompressed baseline. Thus, current budget choices influence the feasible LLA search
and the next cycle's state. Third, LLA
episodes are short but sequential: each step applies a candidate configuration, updates
the realized block retained-size fraction $\bar{\nu}_i$ and hence the budget deviation
$\bar{\nu}_i-R_i$, and contributes to a stability term over a recent reward window.
This temporal credit-assignment structure follows the standard MDP view of sequential
decision making~\citep{sutton2018reinforcement}.

Figure~\ref{fig:reward_flow} shows how this reward signal is generated during one RL
interaction. The action first determines pruning and quantization choices, the temporary
compressed model is scored by the surrogate or logit-MSE proxy, and the resulting reward
combines accuracy retention, compression, budget compliance, stability, and sensitivity.

\begin{figure}[t]
    \centering
    \includegraphics[width=0.9\linewidth]{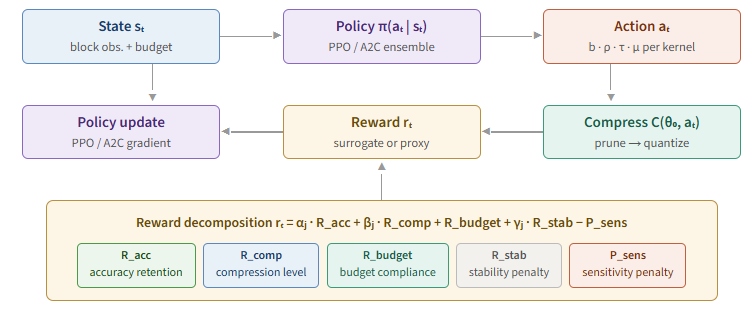}
    \caption{Reward-flow schematic for the RL compression environment. Candidate actions are decoded into compression primitives, evaluated with either the surrogate or logit-MSE proxy, and mapped to the LLA/HLA rewards used for policy updates.}
    \label{fig:reward_flow}
\end{figure}

\paragraph{State Space $\mathcal{S}$}

The index $i \in \{1,\ldots,N\}$ denotes the compressible block currently controlled
by an LLA, after parsing the network into blocks $\{B_i\}_{i=1}^N$. The index
$k \in \{1,\ldots,K_i\}$ denotes a kernel inside that block. Thus, $i/N$ gives the LLA
a normalized indication of the block position in the network, while the remaining state
entries describe the current budget, compression, and reward context for that block.

The \textbf{LLA} observation vector for block $B_i$ has dimension $12 + K_i$:
\begin{itemize}[leftmargin=*]
    \item \textbf{[0]} normalized block index $i/N$; \textbf{[1]} Fisher sensitivity
    score $s_i$; \textbf{[2]} realized block retained-size fraction $\bar{\nu}_i$, initialized
    to $1.0$ before any compression action is applied.
    \item \textbf{[3]} budget deviation $\bar{\nu}_i - R_i$, initialized to $0.0$ at
    reset and recomputed after each candidate action; \textbf{[4]} estimated accuracy
    drop divided by $10$, obtained from the logit-MSE proxy or surrogate.
    \item \textbf{[5]} curriculum-stage indicator, fixed to $1.0$ in the reported
    experiments; \textbf{[6]} HLA-assigned target retained-size fraction $R_i$; \textbf{[7]}
    HLA strategy/preference code $p_i$; \textbf{[8]} block-level allowed drop
    $\Delta_{\mathrm{acc},i}/10$.
    \item \textbf{[9]} constant bias entry $1.0$; \textbf{[10]} surrogate-availability
    flag $\mathbf{1}\{|\mathcal{B}_{\mathrm{rep}}| \geq N_{\mathrm{rep}}^{\min}\}$;
    \textbf{[11]} current-step LLA reward signal $r_t^{\mathrm{LLA}}/100$ in post-transition states, set
    to zero at reset.
    \item \textbf{[12--$12+K_i-1$]} per-kernel occupancy indicators, initialized to
    $0.5$ and updated as the LLA evaluates candidate kernel configurations.
\end{itemize}

The \textbf{HLA} observation vector has dimension $5 + 2N$. At the start of an HLA
episode, its entries are:
\begin{itemize}[leftmargin=*]
    \item \textbf{[0]} previous measured accuracy drop $\Delta\mathcal{A}^{(t-1)}/10$,
    initialized to $0$ when $t=1$; \textbf{[1]} previous realized mean retained-size
    fraction $\bar{\nu}^{(t-1)}$, initialized to $1$ when $t=1$; \textbf{[2]} cycle progress.
    \item \textbf{[3]} global target retained-size fraction $\bar{\nu}^*$; \textbf{[4]} allowed
    accuracy drop $\Delta_{\mathrm{acc}}/10$.
    \item \textbf{[5--$5+2N-1$]} interleaved block descriptors $(s_i,R_i)$ for all
    blocks, so the HLA observes both sensitivity and the current target retained-size proposal
    for every block.
\end{itemize}
After an HLA action is decoded, the next HLA state reuses entries~3 and~4 for dynamic
feedback: entry~3 stores the current global compression gap
$|\bar{\nu}-\bar{\nu}^*|$, and entry~4 stores the normalized immediate reward signal.
This slot reuse is a compact state-encoding choice: before an action, these entries
provide the static global constraints; after the transition, they report how well the
current proposal satisfied those constraints without increasing the state dimension.

\paragraph{Action Space $\mathcal{U}$}
For one block $B_i$, the \textbf{LLA} action is a multi-discrete vector of dimension
$4K_i$ because the block contains $K_i$ kernels and the LLA chooses four discrete decisions
per kernel. Equivalently, the block-level action is the concatenation
$\mathbf{a}_i=(\mathbf{a}_{i,1},\ldots,\mathbf{a}_{i,K_i})$. Across the whole network,
the hierarchy therefore produces $N$ such block-level actions, one for each LLA-controlled
block. For each kernel $k \in \{1,\ldots,K_i\}$, the LLA selects:
\begin{equation}
    \mathbf{a}_{i,k} = \bigl(\delta_{b,i,k},\; \delta_{\rho,i,k},\;
    \delta_{\tau,i,k},\; \delta_{\mu,i,k}\bigr)
    \label{eq:lla_action}
\end{equation}
\noindent where $\delta_{b,i,k}, \delta_{\rho,i,k} \in \{0,\ldots,D\}$ with a shared
resolution $D=14$ for all blocks and kernels (i.e., $15$ discrete levels),
$\delta_{\tau,i,k} \in \{0,1\}$, and $\delta_{\mu,i,k} \in \{0,1,2,3\}$. The indices
are kernel-specific, whereas $D$ is global in the reported implementation. The bitwidth
and keep-ratio are decoded by interpolating between the active bounds for block $B_i$:
\begin{equation}
    b_{i,k} = \mathrm{clip}\!\left(
        b_{\min,i} + \left\lfloor \delta_{b,i,k} \cdot
        \frac{b_{\max} - b_{\min,i}}{D} \right\rfloor,\;
        b_{\min,i},\; b_{\max}
    \right)
    \label{eq:bits_decode}
\end{equation}
\begin{equation}
    \rho_{i,k} = \mathrm{clip}\!\left(
        1 - \left[
            \rho_{\min}^{\mathrm{prune}} + \frac{\delta_{\rho,i,k}}{D}
            \cdot \bigl(\rho_{\max,i} - \rho_{\min}^{\mathrm{prune}}\bigr)
        \right],\;
        1 - \rho_{\max,i},\;
        1 - \rho_{\min}^{\mathrm{prune}}
    \right)
    \label{eq:prune_decode}
\end{equation}
Here $b_{\max}$ and $\rho_{\min}^{\mathrm{prune}}$ are global bounds, while
$b_{\min,i}$ and $\rho_{\max,i}$ are the active block-level bounds supplied by the HLA
\textit{LayerBudget}; if no HLA restriction is used, they reduce to the global
$b_{\min}$ and $\rho_{\max}^{\mathrm{prune}}$. Thus the integer action indices
$\delta_{b,i,k}$ and $\delta_{\rho,i,k}$ are indexed by both block and kernel, whereas
the discretization resolution $D$ is shared.

For the quantization type $\tau_{i,k}$ and granularity $\mu_{i,k}$, the agent operates
in a general per-kernel selection regime. The type index $\delta_{\tau,i,k}$ selects
between INT and FLOAT, while the granularity index $\delta_{\mu,i,k}$ selects among
uniform, log-scale, per-channel, and learned quantization. Global override parameters
may restrict this freedom: when a type override $\tau^*$ is specified,
$\tau_{i,k}=\tau^*$ for all $(i,k)$ and $\delta_{\tau,i,k}$ is disregarded; when a
strategy override $\mu^*$ is specified, $\mu_{i,k}=\mu^*$ globally and
$\delta_{\mu,i,k}$ is disregarded. Without overrides, the agent selects both variables
independently per kernel.

The \textbf{HLA} action is a multi-discrete vector of dimension $2N$. For each block,
it emits a compression-level index $\ell_i$ and a strategy index $\sigma_i$:
\begin{equation}
    \ell_i \in \{0,\ldots,4\}, \qquad
    \sigma_i \in \{\text{quantization},\,\text{pruning},\,\text{auto}\}.
\end{equation}
The five compression levels map to target retained-size fractions $0.35$, $0.30$, $0.25$,
$0.20$, and $0.15$, corresponding approximately to $2.9\times$, $3.3\times$,
$4\times$, $5\times$, and $6.7\times$ compression. The scalar strategy code $p_i$
stored in the LLA observation is the numeric encoding of $\sigma_i$.
The same decoded HLA action forms the block-level \textit{LayerBudget} fields
$(R_i,\rho_{\max,i},b_{\min,i},p_i,\Delta_{\mathrm{acc},i})$: $R_i$ is read directly
from the compression tier, while $\rho_{\max,i}$, $b_{\min,i}$, and
$\Delta_{\mathrm{acc},i}$ are deterministic tier/strategy-dependent bounds used to
restrict the downstream LLA search. Thus the HLA action remains $2N$-dimensional, and
the additional numerical fields are decoded metadata rather than independent action
coordinates.
Sensitivity-aware corrections apply deterministically post-decode: blocks with
$s_i > 0.70$ have their compression level decremented by one tier (lighter budget);
blocks with $s_i < 0.30$ may have it incremented (more aggressive budget).

\paragraph{Transition Function}
In the LLA environment, executing action $\mathbf{a}_t$ applies the decoded compression
configuration via $\mathcal{C}$, after which the next state is constructed by
re-computing the compression ratio, target deviation, and estimated accuracy drop, with
original weights restored from a cached reference copy. In the HLA environment, the
transition updates per-block allocated targets and reflects current global
accuracy--compression feedback. The transition is deterministic given a fixed model
state and action; stochasticity arises solely from the RL policy.

\paragraph{Reward Functions}
For the \textbf{LLA}:
\begin{equation}
    r_t^{\mathrm{LLA}} = \alpha_j \cdot R_{\mathrm{acc}}
        + \beta_j  \cdot R_{\mathrm{comp}}
        + R_{\mathrm{budget}}
        + \gamma_j \cdot R_{\mathrm{stab}}
        - P_{\mathrm{sens}}
    \label{eq:lla_reward}
\end{equation}
where $R_{\mathrm{acc}}$ is a piecewise-linear accuracy retention term (via
surrogate or logit-MSE proxy); $R_{\mathrm{comp}} = 60\cdot(1 - \bar{\nu}_i)$ rewards
higher compression; $R_{\mathrm{budget}} = 20\cdot\exp(-10\cdot|\bar{\nu}_i - R_i|)$
is an exponential compliance term; $R_{\mathrm{stab}} = -2\cdot\mathrm{std}(\mathcal{H}_{t-1})$
penalizes instability using the previous reward window
$\mathcal{H}_{t-1}=\{r_{t-h}^{\mathrm{LLA}}:h=1,\ldots,\min(5,t)\}$, with
$R_{\mathrm{stab}}=0$ when the history is empty; and $P_{\mathrm{sens}}$ is defined in
Eq.~\eqref{eq:psens}.

The per-agent weights $(\alpha_j, \beta_j, \gamma_j)$ are configurable hyperparameters
of the ensemble: $\alpha_j$ controls the relative emphasis on accuracy retention,
$\beta_j$ controls compression aggressiveness, and $\gamma_j$ controls stability
regularization. The ensemble is deliberately configured with diverse $\beta_j$ values so
that the weighted vote in Eq.~\eqref{eq:voting} integrates signals from a conservative
regime, a moderate regime, and an aggressive regime.
\begin{equation}
    P_{\mathrm{sens}} = 12 \cdot s_i \cdot (1 - \bar{\nu}_i),
    \label{eq:psens}
\end{equation}
where $s_i \in [0,1]$ is the Fisher sensitivity score of block $B_i$ and
$\bar{\nu}_i$ is the parameter-weighted mean normalized retained-size fraction across its kernels.

For the \textbf{HLA}:
\begin{equation}
    r_t^{\mathrm{HLA}} = R_{\mathrm{bal}} + R_{\mathrm{sens}} + R_{\mathrm{global}} + R_{\mathrm{dyn}}
    \label{eq:hla_reward}
\end{equation}
where $R_{\mathrm{bal}} = -10\cdot\mathrm{Var}(\{R_i\}_{i=1}^N)$ penalizes
high unweighted heterogeneity in per-block target retained-size fractions;
$R_{\mathrm{sens}}$ encourages sensitivity-aware budget allocation: it rewards
assigning a conservative target ($R_i > 0.3$) to high-sensitivity blocks ($s_i > 0.7$)
with a higher weight than the low-sensitivity case, proportional to
$(R_i-0.3)\times 10$, and independently rewards assigning an aggressive target
($R_i < 0.2$) to low-sensitivity blocks ($s_i < 0.3$), proportional to
$(0.2-R_i)\times 5$; both cases contribute positively to the total reward signal;
$R_{\mathrm{global}} = -50\cdot|\bar{R} - \bar{\nu}^*|$ penalizes deviation of the
mean proposed target from the global target retained-size fraction; and $R_{\mathrm{dyn}}$ is a
cycle-feedback correction term, defined in terms of the previous-cycle mean retained-size fraction
$\bar{\nu}^{(t-1)}$ and accuracy drop $\Delta\mathcal{A}^{(t-1)}$:
\begin{equation}
    R_{\mathrm{dyn}} = \begin{cases}
        -50\cdot\bigl(\Delta\mathcal{A}^{(t-1)} - \Delta_{\mathrm{acc}}\bigr)
        & \text{if } \Delta\mathcal{A}^{(t-1)} > \Delta_{\mathrm{acc}}
          \text{ and } \bar{R} < \bar{\nu}^{(t-1)} \\
        +20
        & \text{if } \Delta\mathcal{A}^{(t-1)} > \Delta_{\mathrm{acc}}
          \text{ and } \bar{R} \geq \bar{\nu}^{(t-1)} \\
        +20
        & \text{if } \bar{\nu}^{(t-1)} > \bar{\nu}^* \text{ and }
          \bar{R} < \bar{\nu}^{(t-1)} \\
        0 & \text{otherwise}
    \end{cases}
    \label{eq:rdyn}
\end{equation}
where $\bar{R}=P^{-1}\sum_i P_iR_i$ is the parameter-weighted mean proposed target
retained-size fraction. The first case penalizes proposing yet-more-aggressive targets after accuracy
has already exceeded the budget; the second rewards relaxing the target in that case;
the third rewards increasing aggressiveness when the previous cycle retained more
storage than the target while still respecting the accuracy budget.

The reward components operate on compatible empirical scales ($R_{\mathrm{acc}}
\in [0,100]$, $R_{\mathrm{comp}} \in [0,60)$, $R_{\mathrm{budget}} \in (0,20]$,
$P_{\mathrm{sens}} \in [0,12)$, $R_{\mathrm{stab}} \leq 0$). Episodes are short
(20 LLA steps, 10 HLA steps) and $\gamma_{\mathrm{RL}}$ is set near 1. The RL
objective at each level is identical in form, but the reward symbol is tier-indexed:
$r_t^{\ell}$ denotes the reward at step $t$ for tier $\ell \in
\{\mathrm{LLA},\mathrm{HLA}\}$, with $r_t^{\mathrm{LLA}}$ defined in
Eq.~\eqref{eq:lla_reward} and $r_t^{\mathrm{HLA}}$ defined in
Eq.~\eqref{eq:hla_reward}:
\begin{equation}
    \pi^* = \arg\max_{\pi} \;\;
    \mathbb{E}_{\pi}\!\left[\, \sum_{t=0}^{T} \gamma_{\mathrm{RL}}^{\,t} \cdot r_t^{\ell} \,\right]
    \label{eq:rl_objective}
\end{equation}

Maximizing Eq.~\eqref{eq:rl_objective} is an RL relaxation of the constrained objective
in Eq.~\eqref{eq:optimization_objective}, not a guarantee of the global optimum of the
integer compression program. In the LLA, $R_{\mathrm{acc}}$ tracks accuracy retention and
$R_{\mathrm{comp}}$ tracks storage reduction, while $R_{\mathrm{budget}}$ constrains the
block to the HLA-assigned target $R_i$. In the HLA, $R_{\mathrm{global}}$ directly
penalizes deviation from the global target retained-size fraction $\bar{\nu}^*$. This mirrors
other combinatorial optimization settings where policy gradients provide a practical
search heuristic for large discrete spaces~\citep{bello2016neural}.

The stability reward uses a rolling window from the environment's internal reward
history. In principle this can be made exactly Markovian by augmenting the state with the
full reward window; HiReLC instead uses a compact state containing the current normalized
reward signal, avoiding additional sample complexity while retaining a stabilizing bias
toward smooth decisions~\citep{nath2021revisiting,hessel2019inductive}.

\subsection{Architecture-Agnostic Layer Abstraction}

Each model is parsed into an ordered list of compressible blocks via a layer abstraction
module. A block $B_i$ is defined as a contiguous group of $K_i$ parametric
\texttt{Linear} or \texttt{Conv2d} modules that are logically associated (e.g., the QKV
projection, attention output projection, and MLP feed-forward projections in a
Transformer encoder layer; or the depthwise and pointwise convolutions within a
bottleneck residual unit). Each kernel is described by a \textit{KernelConfig} object
encoding the four compression decisions of Section~\ref{sec:compression_function}.
The aggregated block-level configuration, referred to as a \textit{LayerConfig},
maintains per-kernel fields alongside summary statistics (average bitwidth, average
keep-ratio, normalized retained-size fraction via Eq.~\eqref{eq:compression_ratio},
and effective compression ratio via Eq.~\eqref{eq:global_cr}) that serve as state
features for higher-level agents. This abstraction decouples the RL
environment from the underlying architecture, ensuring that the same agent interfaces,
reward signals, and training procedures apply to any target network.

\subsection{Layer Sensitivity Estimation}

A scalar sensitivity score $s_i \in [0,1]$ is computed for each block $B_i$ via
empirical Fisher information before RL training. For a mini-batch $\mathcal{D}_s$ of
$n_s$ samples:
\begin{equation}
    \hat{F}(\theta_j) = \frac{1}{n_s}
    \sum_{(\mathbf{x},\, y)\,\in\,\mathcal{D}_s}
    \left( \frac{\partial \mathcal{L}(\mathbf{x}, y;\, \boldsymbol{\theta})}
                {\partial \theta_j} \right)^{\!2}
    \label{eq:fisher}
\end{equation}
where $\mathcal{L}(\mathbf{x}, y;\,\boldsymbol{\theta})=\mathrm{CrossEntropy}(f_{\boldsymbol{\theta}}(\mathbf{x}),\,y)$ is the cross-entropy loss evaluated at the model output.

\noindent The block-level score is the sum over $B_i$'s parameters, max-normalized:
\begin{equation}
    s_i = \frac{\displaystyle\sum_{\theta_j \in B_i} \hat{F}(\theta_j)}
               {\displaystyle\max_{i'} \sum_{\theta_j \in B_{i'}} \hat{F}(\theta_j)}
    \label{eq:sensitivity_score}
\end{equation}

\noindent These scores $\{s_i\}$ serve two roles: (i) they initialize the LLA
observation vectors and enter the sensitivity penalty $P_{\mathrm{sens}}$
(Eq.~\eqref{eq:psens}), where a block with high $s_i$ incurs a larger penalty per
unit of compression applied; and (ii) they inform HLA budget allocation via the
sensitivity-aware action correction and the $R_{\mathrm{sens}}$ reward term.

\subsection{Hierarchical Agent Architecture}
\label{sec:hierarchical_agents}

The framework employs a two-tier hierarchical RL structure, motivated by the distinct
temporal and spatial scales of the compression decision process.

\begin{figure}[t]
    \centering
    \includegraphics[width=1\linewidth]{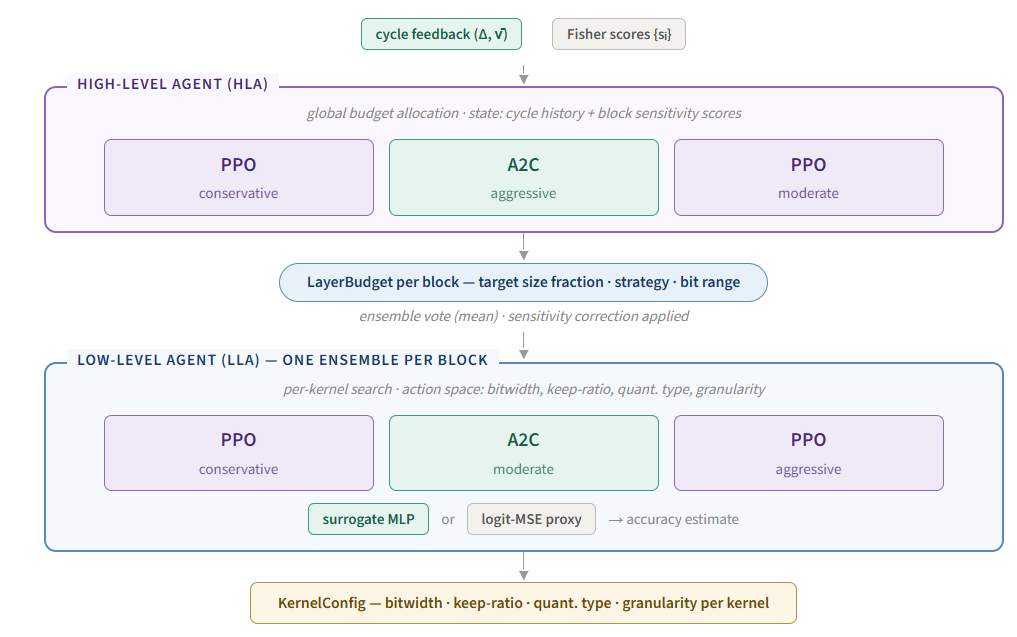}
    \caption{Two-tier controller used in HiReLC. The HLA ensemble proposes per-block \textit{LayerBudget} objects, LLA ensembles search within each budget over per-kernel actions, and the weighted vote emits the final \textit{KernelConfig} assignments.}
    \label{fig:hierarchical_controller}
\end{figure}

\subsubsection{Low-Level Agents (LLA)}
\label{sec:lla}

Each block $B_i$ is governed by an \textbf{Ensemble LLA} of $n_a$ parallel RL policies
combining PPO and A2C with gradient norm clipping~\citep{schulman2017proximal,mnih2016asynchronous}.
PPO is used for stable clipped policy updates under aggressive compression, while A2C
provides fast synchronous actor--critic exploration in the same multi-discrete action
space. At configuration extraction, a
uniform-weighted vote produces the consensus action:
\begin{equation}
    \mathbf{a}_i^* =
    \operatorname{round}\!\left(\, \sum_{j=1}^{n_a} w_j \cdot \mathbf{a}_i^{(j)} \,\right),
    \qquad w_j = 1/n_a
    \label{eq:voting}
\end{equation}

\noindent The objective is a robust consensus action obtained by averaging across $n_a$
agents trained with deliberately heterogeneous reward hyperparameters $(\alpha_j,
\beta_j, \gamma_j)$. By configuring the ensemble with $\beta_j$ values spanning conservative, moderate,
and aggressive compression emphases, the weighted vote blends diverse policies and reduces
the variance of the final compression decision relative to any single policy.

\paragraph{Agent Weight Policy}
Weights $\{w_j\}$ can be fixed uniformly ($w_j = 1/n_a$, adopted in our experiments),
set as tuned hyperparameters, or updated online via a confidence-based rule:
\begin{equation}
    w_j^{(t+1)} \;=\; \frac{\exp\!\left(\lambda \cdot \overline{r}_j^{(t)}\right)}
                            {\displaystyle\sum_{j'=1}^{n_a}
                             \exp\!\left(\lambda \cdot \overline{r}_{j'}^{(t)}\right)},
    \label{eq:weight_update}
\end{equation}
where $\overline{r}_j^{(t)} = \frac{1}{H}\sum_{h=1}^{H} r_{t-h+1}^{(j)}$
is the sliding-window average of the most recent $H$ rewards for agent $j$ and
$\lambda > 0$ is a temperature parameter. Evaluating this online update rule is left as
future work.

The consensus action is decoded into $(b_{i,k},\,\rho_{i,k},\,\tau_{i,k},\,\mu_{i,k})$
for all $K_i$ kernels in block $B_i$, subject to global type and granularity override parameters.

An ablation on DeiT-Small/CIFAR-100 confirms that this mixed PPO/A2C ensemble gives a
more favorable first-cycle trade-off than a homogeneous A2C ensemble: [PPO, A2C, PPO]
achieves 6.38$\times$ compression with 3.28\% accuracy drop, whereas [A2C, A2C, A2C]
achieves 6.63$\times$ but with a larger 7.19\% drop.

\subsubsection{High-Level Agents (HLA)}
\label{sec:hla}

The \textbf{Ensemble HLA} of $n_h$ agents (PPO and A2C) with heterogeneous reward
weights allocates a \textit{LayerBudget} $\mathcal{B}_i$ per block specifying $R_i$,
$\rho_{\max,i}$, $b_{\min,i}$, $\sigma_i$, $p_i$, and $\Delta_{\mathrm{acc},i}$.
Numerical budget fields are aggregated by arithmetic mean; the preferred strategy
$\sigma_i$ is the majority-vote mode of per-agent proposals. For cycles $t > 1$, HLA
environments receive real-world feedback $(\Delta\mathcal{A}^{(t-1)},
\bar{\nu}^{(t-1)})$ and the ensemble is re-trained for a short burst of additional
timesteps to adapt the budget allocation policy to observed LLA behavior. The final
cycle uses deterministic (greedy) HLA actions.

The hierarchy is essential rather than cosmetic: removing the HLA and using uniform
budgets leaves the LLA to over-compress, yielding 6.80\% accuracy drop at 8.87$\times$
CR on DeiT-Small/CIFAR-100, compared with 1.72\% at 6.23$\times$ for the full system.
This shows that HLA feedback primarily regulates budget feasibility, while the LLA
ensemble performs fine-grained per-kernel search.

\subsection{Compression Primitives}
\label{sec:compression_primitives}

\subsubsection{Structured Pruning}

For kernel $k$ in block $B_i$, let $\mathbf{W}$ denote its weight tensor and let
$C_{\mathrm{out}}$ be the number of output channels of that tensor (rows for a linear
layer and filters for a convolution). The keep-ratio $\rho_{i,k}\in(0,1]$ specifies the
fraction of output channels retained, so the number of retained channels is
$K_{\mathrm{keep}}=\lfloor\rho_{i,k}C_{\mathrm{out}}\rfloor$. We score each output
channel by an $\ell_2$-norm importance value $\eta_j$ and keep the
$K_{\mathrm{keep}}$ most important channels:
\begin{equation}
    \widetilde{\mathbf{W}}_{\mathrm{pruned}} = \mathbf{m} \odot \mathbf{W},
    \qquad
    m_j = \mathbf{1}\!\left[
        \operatorname{rank}_{\downarrow}(\eta_j) \leq
        \left\lfloor \rho_{i,k} \cdot C_{\mathrm{out}} \right\rfloor
    \right]
    \label{eq:pruning_mask}
\end{equation}
\noindent Here $\widetilde{\mathbf{W}}_{\mathrm{pruned}}$ is the masked weight tensor,
$\mathbf{m}$ is a binary channel mask broadcast over the corresponding weights,
$\odot$ denotes elementwise multiplication, $m_j\in\{0,1\}$ is the mask value for output
channel $j$, and $\operatorname{rank}_{\downarrow}(\eta_j)=1$ denotes the channel with
the largest importance score. For a convolutional kernel, for example,
$\eta_j=\|\mathbf{W}_{j,:,:,:}\|_2$; for a linear layer, $\eta_j$ is the norm of row
$j$. Intuitively, Eq.~\eqref{eq:pruning_mask} keeps the channels whose filters/rows have
the largest energy and zeros the remaining channels. A keep-ratio of $\rho_{i,k}=0.25$
therefore retains approximately the top quarter of output channels in kernel $(i,k)$.

\noindent Masks are re-applied after every gradient step during fine-tuning. In the
implementation used for the reported experiments, pruning is represented by fixed masks
for stable recovery training; the reported model-size reductions count the retained
channels under a packed structured representation. On hardware that does not physically
remove masked channels, the same configuration should be interpreted as a
storage-compression result rather than a guaranteed latency speedup. A deployment pass
could materialize the selected channel subsets into narrower tensors, but measured latency
is left outside the present objective.

\subsubsection{Mixed-Precision Quantization}

Pruning is applied first; quantization follows. For $\tau_{i,k} = \mathrm{INT}$
(symmetric), each scalar weight $w$ in the pruned kernel is quantized as:
\begin{equation}
\begin{aligned}
    \tilde{w} &=
    \mathrm{clip}\!\left(
        \operatorname{round}\!\left( \frac{w}{s_{i,k}} \right),\;
        q_{\min},\; q_{\max}
    \right) \cdot s_{i,k}, \\
    s_{i,k} &= \frac{\max(|\widetilde{\mathbf{W}}_{\mathrm{pruned}}|)}{q_{\max}},
    \qquad
    q_{\max} = 2^{b_{i,k}-1} - 1,
    \qquad
    q_{\min} = -q_{\max}.
\end{aligned}
    \label{eq:quantization}
\end{equation}
\noindent Here $\tilde{w}$ is the de-quantized value used by the compressed model,
$s_{i,k}$ is the quantization scale for kernel $(i,k)$, and
$[q_{\min},q_{\max}]$ is the signed integer range determined by the bitwidth
$b_{i,k}$. Equation~\eqref{eq:quantization} first maps $w$ to the nearest integer grid
point by dividing by $s_{i,k}$ and rounding, then clips the result to the representable
integer range, and finally multiplies by $s_{i,k}$ to return to the real-valued weight
scale. Lower bitwidths reduce $q_{\max}$ and therefore use a coarser grid, increasing
quantization error but reducing storage. The displayed scale is the per-tensor uniform
case; per-channel variants apply the same operation independently to each output channel
with a channel-specific scale.

\noindent For $\tau_{i,k} = \mathrm{FLOAT}$, asymmetric quantization is applied (scale
from the full $[\min,\max]$ range). Granularity $\mu_{i,k}$ selects among four
operationally distinct schemes: per-tensor uniform (\textit{uniform}); log-domain
per-tensor, where weights are log-transformed, uniformly quantized, then exponentiated
back with sign restoration (\textit{log}); per-output-channel uniform or asymmetric
(\textit{per-channel}); and per-channel sign-magnitude (\textit{learned}), where one
explicit bit per output channel stores the sign and the remaining $(b_{i,k}-1)$ bits
uniformly quantize the absolute magnitude:
\begin{equation}
    \tilde{w}_{j,u} = \mathrm{sgn}(w_{j,u}) \cdot
        \mathrm{clip}\!\left(
            \operatorname{round}\!\left(\frac{|w_{j,u}|}{s_j^+}\right),
            0,\; 2^{b_{i,k}-1}-1
        \right) \cdot s_j^+,
    \quad
    s_j^+ = \frac{\max_u|w_{j,u}|}{2^{b_{i,k}-1}-1}
    \label{eq:sign_magnitude}
\end{equation}
where $j$ indexes the output channel and $u$ indexes scalar weights within that channel.

The normalized retained-size fraction of a compressed kernel is:
\begin{equation}
    \nu_{i,k} = \frac{b_{i,k}}{b_0} \cdot \rho_{i,k}
    \label{eq:compression_ratio}
\end{equation}
where $b_0 = 32$ and $\nu_{i,k}\in(0,1]$. This quantity is an effective
parameter-storage fraction: it combines low-bit storage with the retained-channel fraction
and does not include hardware-specific indexing overheads, activation memory, or kernel
launch effects. Let $P_{i,k}$ be the FP32 parameter count of kernel $(i,k)$,
$P_i=\sum_{k=1}^{K_i}P_{i,k}$, and $P=\sum_{i=1}^{N}P_i$. The global compression ratio is:
\begin{equation}
    \mathcal{R}(\mathbf{c}) \equiv \mathrm{CR} = \frac{1}{\bar{\nu}},
    \qquad
    \bar{\nu} = \frac{1}{P} \sum_{i=1}^{N} P_i\bar{\nu}_i,
    \qquad
    \bar{\nu}_i = \frac{1}{P_i} \sum_{k=1}^{K_i} P_{i,k}\nu_{i,k}
    \label{eq:global_cr}
\end{equation}
\noindent We report CR as the achieved effective parameter-storage compression ratio. The
associated model-size reduction is
$\mathrm{MSR}=1-\bar{\nu}=1-1/\mathrm{CR}$; for example, $\mathrm{CR}=4\times$
corresponds to a $75\%$ model-size reduction. Thus $R_i$, $\bar{R}$, and
$\bar{\nu}^*$ are retained-size fractions, whereas $R_{\mathrm{target}}$ and CR are
compression-ratio quantities.

\paragraph{Post-Compression Fine-Tuning}
Fine-tuning is performed directly on the compressed weights $\boldsymbol{\theta}'$
produced by $\mathcal{C}$. At each optimizer step, AdamW updates the weights via
standard backpropagation, after which the pruning masks from
Eq.~\eqref{eq:pruning_mask} are re-applied to every masked weight tensor, preserving
the sparsity pattern throughout the recovery phase. A cosine-annealing learning rate
schedule and patience-based early stopping are used to prevent overfitting. This
in-place refinement lets weights adapt to quantization error without allowing pruned
channels to regrow, following the common practice of maintaining fixed masks during
post-pruning recovery~\citep{han2015learning,anwar2017structured}.

\subsection{Surrogate Model and Logit-MSE Proxy}
\label{sec:surrogate}

\paragraph{Logit-MSE Proxy}
When fewer than $N_{\mathrm{rep}}^{\min}$ samples are available in the replay buffer,
the learned surrogate is not yet reliable. In this cold-start regime, the LLA uses the
logit-MSE below as a cheap proxy for accuracy degradation inside the accuracy-retention
reward term $R_{\mathrm{acc}}$ in Eq.~\eqref{eq:lla_reward}. Here
$\boldsymbol{\theta}_0$ denotes the fine-tuned uncompressed baseline used to cache
reference logits:
\begin{equation}
    \mathrm{MSE}_{\mathrm{logit}} =
    \frac{1}{|\mathcal{X}_c|}
    \sum_{\mathbf{x} \in \mathcal{X}_c}
    \left\| f_{\boldsymbol{\theta}'}(\mathbf{x})
          - f_{\boldsymbol{\theta}_0}(\mathbf{x}) \right\|^2
    \label{eq:logit_mse}
\end{equation}
where $f_{\boldsymbol{\theta}_0}(\mathbf{x})$ and
$f_{\boldsymbol{\theta}'}(\mathbf{x})$ are the output logit vectors of the original and
compressed model respectively, and $\mathcal{X}_c$ is a cached calibration set of a
fixed number of mini-batches from the training loader. A smaller
$\mathrm{MSE}_{\mathrm{logit}}$ indicates that the compressed model preserves the
original model's predictions more closely, so it receives a higher accuracy-retention
reward during early LLA search.

\paragraph{Surrogate Model}
Trainers may include a surrogate \emph{warm-up phase} prior to the
first RL cycle: a set of random compression configurations are generated, each is applied
to a temporary model copy and fine-tuned for a single epoch, and the resulting
$(\mathbf{c}, \mathcal{A})$ pair is added to the replay buffer. The surrogate MLP is
then trained on this initial dataset before the first RL cycle begins.

Let $\mathcal{B}_{\mathrm{rep}}$ denote the replay buffer of collected $(\mathbf{c}, \mathcal{A})$ pairs.
Let $N_{\mathrm{rep}}^{\min}$ be the minimum replay-buffer size required to activate the
surrogate; in this paper, we set $N_{\mathrm{rep}}^{\min}=3$ so that the surrogate can be
used after several cycle-level measurements, while CNN warm-up experiments add larger
random pools before cycle 1. Once $|\mathcal{B}_{\mathrm{rep}}| \geq N_{\mathrm{rep}}^{\min}$,
the surrogate replaces the logit-MSE proxy for reward shaping, but every reported final
accuracy is still obtained by fine-tuning and evaluating the compressed model.

Once $|\mathcal{B}_{\mathrm{rep}}| \geq N_{\mathrm{rep}}^{\min}$, a lightweight MLP
$\hat{f}: \mathbb{R}^{5K_{\mathrm{tot}}} \to \mathbb{R}$ (ReLU activations, dropout) replaces the
logit-MSE proxy. Its input encodes five values per kernel: normalized bitwidth
($b_{i,k}/10$), keep-ratio $\rho_{i,k}$, type ($0$ for INT, $1$ for FLOAT),
granularity ($\{0,\,0.33,\,0.67,\,1.0\}$ for \{uniform, log, per-channel, learned\}),
and a joint compression indicator in $[0.5, 1.0]$ (set to $1.0$ when both pruning and
sub-8-bit quantization are active, else $0.5$). Training minimizes:
\begin{equation}
    \mathcal{L}_{\mathrm{surr}} =
    \frac{1}{|\mathcal{B}_{\mathrm{rep}}|}
    \sum_{(\mathbf{c},\, a) \in \mathcal{B}_{\mathrm{rep}}}
    \left(\hat{f}(\mathbf{c}) - a\right)^2
    \label{eq:surrogate_loss}
\end{equation}
against ground-truth post-fine-tuning accuracy values $a$ collected in the
replay buffer $\mathcal{B}_{\mathrm{rep}}$. The surrogate is re-trained at the end of
each cycle as new $(\mathbf{c}^{(t)}, \mathcal{A}^{(t)})$ pairs are added.

The surrogate/proxy split is crucial for tractability: evaluating every candidate by full
post-compression fine-tuning would require hundreds of fine-tuning runs per experiment.
In our implementation, random warm-up samples use a single epoch and RL steps use either
logit-MSE or the MLP prediction, reducing search cost while reserving full fine-tuning for
selected cycle-level configurations. Because the surrogate can be noisy for compact CNNs,
we interpret it as a coarse ranking signal and report its MAE separately rather than using
it as the final performance estimate.

\subsection{Active Learning Compression Loop}
\label{sec:active_learning}

The full framework is orchestrated through an iterative active learning loop that
interleaves RL policy optimization with real model fine-tuning. For each cycle
$t \in \{1, \ldots, T\}$:

\begin{enumerate}
    \item \textbf{HLA Feedback Update.} The HLA environments are updated with the
          accuracy drop $\Delta\mathcal{A}^{(t-1)}$ and mean retained-size fraction
          $\bar{\nu}^{(t-1)}$ measured in the previous cycle, using
          $\Delta\mathcal{A}^{(0)}=0$ and $\bar{\nu}^{(0)}=1$ in the first cycle.

    \item \textbf{HLA Policy Adaptation.} For $t > 1$, the HLA agents are re-trained for
          a short burst of additional timesteps to adapt to observed LLA behavior.

    \item \textbf{Budget Allocation.} The HLA ensemble proposes a compression budget
          $\mathcal{B}_i^{(t)}$ for each block via weighted ensemble voting
          (Eq.~\eqref{eq:voting}). Early cycles use stochastic (exploratory) HLA actions;
          the final cycle uses deterministic exploitation.

    \item \textbf{LLA Training.} Each block's LLA ensemble is trained (or continues
          from the previous cycle's state) under the updated budget constraints.

    \item \textbf{Configuration Extraction.} The trained LLA agents output consensus
          configurations $\{c_{i,k}^{(t)}\}$ via ensemble voting
          (Eq.~\eqref{eq:voting}).

    \item \textbf{Compression Application.} The configurations are applied via
          $\mathcal{C}(\boldsymbol{\theta}_0, \mathbf{c}^{(t)})$: pruning
          (Eq.~\eqref{eq:pruning_mask}) followed by quantization
          (Eq.~\eqref{eq:quantization}) for every kernel.

    \item \textbf{Fine-Tuning.} The compressed model is fine-tuned via AdamW with
          cosine-annealing LR and patience-based early stopping. Pruning masks are
          re-applied after every gradient step to preserve the sparsity pattern
          throughout the recovery phase.

    \item \textbf{Surrogate Update.} The pair $(\mathbf{c}^{(t)},\, \mathcal{A}^{(t)})$
          is added to $\mathcal{B}_{\mathrm{rep}}$, and the surrogate is re-trained if
          the activation condition is met.

    \item \textbf{Best Configuration Tracking.} If $\mathcal{A}^{(t)}$ exceeds the
          current global best, the configuration and model checkpoint are saved.
\end{enumerate}


\section{Experiments}
\label{sec:exp_descriptions}

We conduct seven experiments across ViT and CNN architectures on CIFAR-10/100 and Tiny ImageNet. Experiments 1--5 use INT uniform symmetric quantization; Experiments 6--7 ablate mixed INT/FLOAT with log-scale granularity. The first five experiments cover the default INT setting across transformer and CNN families, while the final two isolate the effect of mixed INT/FLOAT quantization with log-scale granularity.

\subsubsection*{Experiment 1: DeiT-Small on CIFAR-100}
We employ the \texttt{deit\_small\_patch16\_224} model \citep{touvron2021training} configured with 48 kernels across 12 blocks. Each block undergoes simultaneous structured pruning and INT quantization. The implementation supports surrogate warm-up through randomly sampled compression configurations; when the replay buffer is below the activation threshold, the LLA instead uses the logit-MSE proxy described in Section~\ref{sec:surrogate}.

\subsubsection*{Experiment 2: CLIP ViT-B/32 on CIFAR-100}
\texttt{vit\_base\_patch32\_clip\_224} \citep{radford2021learning} fine-tuned on CIFAR-100 (48 kernels, 12 blocks). Same surrogate schedule as Experiment 1. This setting tests compressibility of vision-language representations on a single-dataset classification task.

\subsubsection*{Experiment 3: CLIP ViT-B/32 on CIFAR-10}
\texttt{vit\_base\_patch32\_clip\_224} \citep{radford2021learning} fine-tuned on CIFAR-10 (10 classes), with 20 fine-tuning epochs per cycle; the surrogate is disabled and the logit-MSE proxy is used throughout. Baseline accuracy is 89.22\%; after one compression cycle the model reaches 93.05\%, a gain attributable to compression-as-regularization correcting overfitting in an over-parameterized model.

\subsubsection*{Experiment 4: ResNet18 on Tiny ImageNet}
\texttt{resnet18} \citep{he2016deep} on Tiny ImageNet (200 classes, $64\times64$), with 17 kernels across 5 block groups. A 100-sample warm-up pool initializes the surrogate from cycle 1. Heterogeneous residual structure yields a surrogate MAE of 15.20\% ($\pm$2.52\%).

\subsubsection*{Experiment 5: MobileNetV2 on Tiny ImageNet}
\texttt{mobilenetv2\_100} \citep{sandler2018mobilenetv2} with 51 kernels across 13 inverted residual blocks; warm-up pool as in Experiment 4. Surrogate MAE of 10.37\% ($\pm$0.46\%), reflecting the smoother accuracy--compression landscape of the uniform-width design. The already-compact architecture leaves little room for compression without accuracy loss.

\subsubsection*{Experiment 6: DeiT-Small on CIFAR-100 -- Mixed INT/FLOAT + Log Granularity}
Identical to Experiment 1 but with the LLA free to select INT or FLOAT per kernel and log-scale granularity applied uniformly. Surrogate disabled; logit-MSE proxy used throughout.

\subsubsection*{Experiment 7: DeiT-Base on CIFAR-100 -- Mixed INT/FLOAT + Log Granularity}
\texttt{deit\_base\_patch16\_224} \citep{touvron2021training} ($\approx3\times$ the parameters of DeiT-Small) under the same mixed/log configuration as Experiment 6.

\paragraph{Reproducibility protocol}
For multi-seed reporting, DeiT-Small/CIFAR-100 is evaluated over three independent seeds
($42$, $123$, and $456$), with agent-level PPO/A2C initializations offset by ensemble
index. Single-run rows in the main tables are retained for direct comparison with the
original experiment numbering, while Table~\ref{tab:multiseed} reports the corresponding
seed-averaged uncertainty. All accuracy drops are computed relative to the uncompressed
baseline evaluated under the same fine-tuning and validation protocol.

\section{Quantitative Results}
\label{sec:results}

Tables~\ref{tab:main_results}--\ref{tab:extra_results} summarize all seven experiment outcomes. In these tables, CR denotes the achieved effective parameter-storage compression ratio in Eq.~\eqref{eq:global_cr}, and MSR denotes model-size reduction $1-1/\mathrm{CR}$. Table~\ref{tab:multiseed} provides uncertainty quantification for the DeiT-Small/CIFAR-100 setting.

\begin{table}[htbp]
    \caption{HiReLC compression results. Drop is signed top-1 change relative to baseline ($+$ = loss, $-$ = gain). CR is the achieved effective parameter-storage compression ratio; MSR is model-size reduction, $\mathrm{MSR}=1-1/\mathrm{CR}$. Experiments 6--7 use mixed INT/FLOAT + log granularity.}
    \label{tab:main_results}
    \begin{center}
        \renewcommand{\arraystretch}{1.3}
        \resizebox{\linewidth}{!}{%
        \begin{tabular}{lllllll}
            \textbf{Experiment} & \textbf{Model} & \textbf{Dataset} & \textbf{Baseline} & \textbf{Final} & \textbf{Drop} & \textbf{CR / MSR} \\
            \hline \\
            1 & DeiT-Small    & CIFAR-100     & 86.88\% & 85.16\% & $+$1.72\%  & 6.23$\times$ / 83.9\% \\
            2 & CLIP ViT-B/32 & CIFAR-100     & 57.50\% & 56.95\% & $+$0.55\%  & 6.63$\times$ / 84.9\% \\
            3 & CLIP ViT-B/32 & CIFAR-10      & 89.22\% & 93.05\% & $-$3.83\%  & 6.64$\times$ / 84.9\% \\
            4 & ResNet18      & Tiny ImageNet & 75.16\% & 70.94\% & $+$4.22\%  & 5.99$\times$ / 83.3\% \\
            5 & MobileNetV2   & Tiny ImageNet & 76.56\% & 70.94\% & $+$5.62\%  & 6.59$\times$ / 84.8\% \\
            \hline \\
            6 & DeiT-Small    & CIFAR-100     & 87.42\% & 86.25\% & $+$1.17\%  & 6.27$\times$ / 84.0\% \\
            7 & DeiT-Base     & CIFAR-100     & 90.70\% & 88.59\% & $+$2.11\%  & 6.72$\times$ / 85.1\% \\
        \end{tabular}
        }
    \end{center}
\end{table}

\begin{table}[htbp]
    \caption{Multi-seed reproducibility on DeiT-Small/CIFAR-100 over seeds 42, 123, and 456. Values are mean $\pm$ standard deviation across independent compression runs; CR denotes achieved effective parameter-storage compression ratio.}
    \label{tab:multiseed}
    \begin{center}
        \renewcommand{\arraystretch}{1.25}
        \begin{tabular}{lllll}
            \textbf{Setting} & \textbf{Baseline} & \textbf{Final} & \textbf{Drop} & \textbf{CR} \\
            \hline \\
            HiReLC & 87.86\% & 85.65\% & $2.21\pm0.50$\% & $6.30\pm0.19\times$ \\
        \end{tabular}
    \end{center}
\end{table}

\begin{table}[htbp]
    \caption{Per-experiment surrogate and HLA strategy metrics. Experiments 3, 6, and 7 use the logit-MSE proxy with the surrogate disabled.}
    \label{tab:extra_results}
    \begin{center}
        \renewcommand{\arraystretch}{1.3}
        \resizebox{\linewidth}{!}{
          \begin{tabular}{llllll}
              \textbf{Experiment} & \textbf{Kernels} & \textbf{Blocks} & \textbf{Surr.\ MAE} & \textbf{Dominant HLA Strategy} & \textbf{Budget Comp.} \\
              \hline \\
              1 & 48 & 12 & 3.57\% (cycle 4 only)   & Auto (75.0\%), Pruning (16.7\%), Quant (8.3\%)          & 59.8\% \\
              2 & 48 & 12 & N/A (logit-MSE, all)   & Pruning (50.0\%), Auto (25.0\%), Quant (25.0\%)         & 59.2\% \\
              4 & 17 & 5  & 15.20\% ($\pm$2.52\%)   & Quantization (80.0\%), Pruning (20.0\%)                  & 60.9\% \\
              5 & 51 & 13 & 10.37\% ($\pm$0.46\%)   & Auto (53.8\%), Quant (30.8\%), Pruning (15.4\%)         & 62.4\% \\
              \hline \\
              3 & 48 & 12  & N/A (disabled)             & Quantization (33.33\%), Auto (50.00\%), Pruning (16.67\%)                  & N/A \\
              6 & 48 & 12 & N/A (disabled)           & Auto (66.7\%), Pruning (16.7\%), Quant (16.7\%)         & 58.3\% \\
              7 & 48 & 12 & N/A (disabled)           & Pruning (66.7\%), Auto (25.0\%), Quant (8.3\%)          & 59.7\% \\
          \end{tabular}
        }
    \end{center}
\end{table}

\textbf{Experiment 1 (DeiT-Small / CIFAR-100):} From 86.88\%, HiReLC yields 85.16\% (a 1.72\% drop) at 6.23$\times$ CR and 83.9\% MSR, with average global sparsity of 19.18\%. Best result in cycle 3; per-cycle drops: 3.52\%, 2.34\%, 1.72\%, 2.81\%. Figure~\ref{fig:heatmap} summarizes the learned per-block allocation for this experiment.

\begin{figure}[t]
    \centering
    \includegraphics[width=1\linewidth]{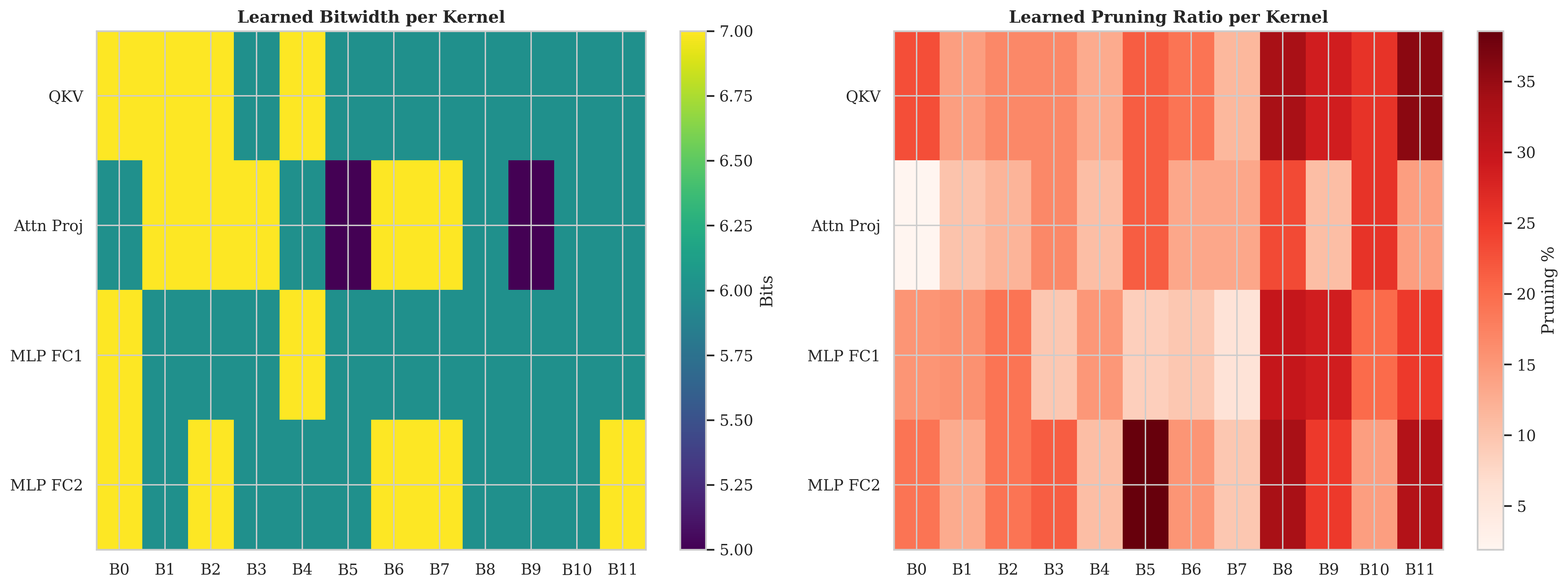}
    \caption{Per-block compression allocation heatmap for Experiment 1: rows are kernel types (QKV, AttnProj, FC1, FC2), columns are transformer blocks, and darker cells indicate higher compression.}
    \label{fig:heatmap}
\end{figure}

\noindent The allocation pattern in Figure~\ref{fig:heatmap} supports the quantitative
result of Experiment 1: HiReLC does not reach its compression ratio by applying a
uniform policy across the transformer. Instead, the controller preserves selected early
attention projections and mid-depth MLP components while shifting more aggressive
compression toward later blocks. This behavior is consistent with the intended role of
sensitivity-aware budgeting: components that appear more fragile retain capacity,
whereas more redundant later-block components absorb a larger share of the pruning and
lower-bit assignments.

\textbf{Experiment 2 (CLIP ViT-B/32 / CIFAR-100):} Smallest drop of any experiment: 0.55\% at 6.63$\times$ CR. The HLA favors pruning (50\%), consistent with substantial structural redundancy in CLIP attention heads. Best result in cycle 2; per-cycle drops: 3.12\%, 0.55\%, 0.86\%, 3.36\%.

\textbf{Experiment 3 (CLIP ViT-B/32 / CIFAR-10):} Accuracy improves from 89.22\% to 93.05\% (a 3.83\% gain) at 6.64$\times$ CR after a single cycle. The compression-plus-fine-tuning acts as a regularizer, correcting overfitting of the large-capacity CLIP model to the 10-class task.

\textbf{Experiment 4 (ResNet18 / Tiny ImageNet):} From 75.16\%, the final model reaches 70.94\% (a 4.22\% drop) at 5.99$\times$ CR. Monotonically improving per-cycle drops (11.72\%, 5.08\%, 4.53\%, 4.22\%). The HLA selects quantization-centric strategies (80\%), reflecting ResNet's pruning sensitivity.

\textbf{Experiment 5 (MobileNetV2 / Tiny ImageNet):} From 76.56\%, the model reaches 70.94\% (a 5.62\% drop) at 6.59$\times$ CR. Best in cycle 2. Figure~\ref{fig:budget_alloc_all_cycles} shows the HLA progressively shifting toward pruning in mid-to-late blocks across cycles.

\textbf{Experiment 6 (DeiT-Small / CIFAR-100 / Mixed+Log):} Achieves 86.25\% (a 1.17\% drop) at 6.27$\times$ CR, a lower drop than the INT uniform Experiment 1 (1.72\%) at a comparable CR, suggesting per-kernel type flexibility aids accuracy recovery. Best in cycle 2; per-cycle drops: 4.77\%, 1.17\%, 2.58\%, 1.48\%.

\textbf{Experiment 7 (DeiT-Base / CIFAR-100 / Mixed+Log):} From a 90.70\% baseline, one cycle yields 88.59\% (a 2.11\% drop) at 6.72$\times$ CR and 85.1\% MSR, the highest CR and MSR in this work. The HLA strongly favors pruning (66.7\%), consistent with greater structural redundancy in the larger model.

\begin{figure}[H]
    \centering
    \begin{subfigure}{0.88\linewidth}
        \centering
        \includegraphics[width=\linewidth]{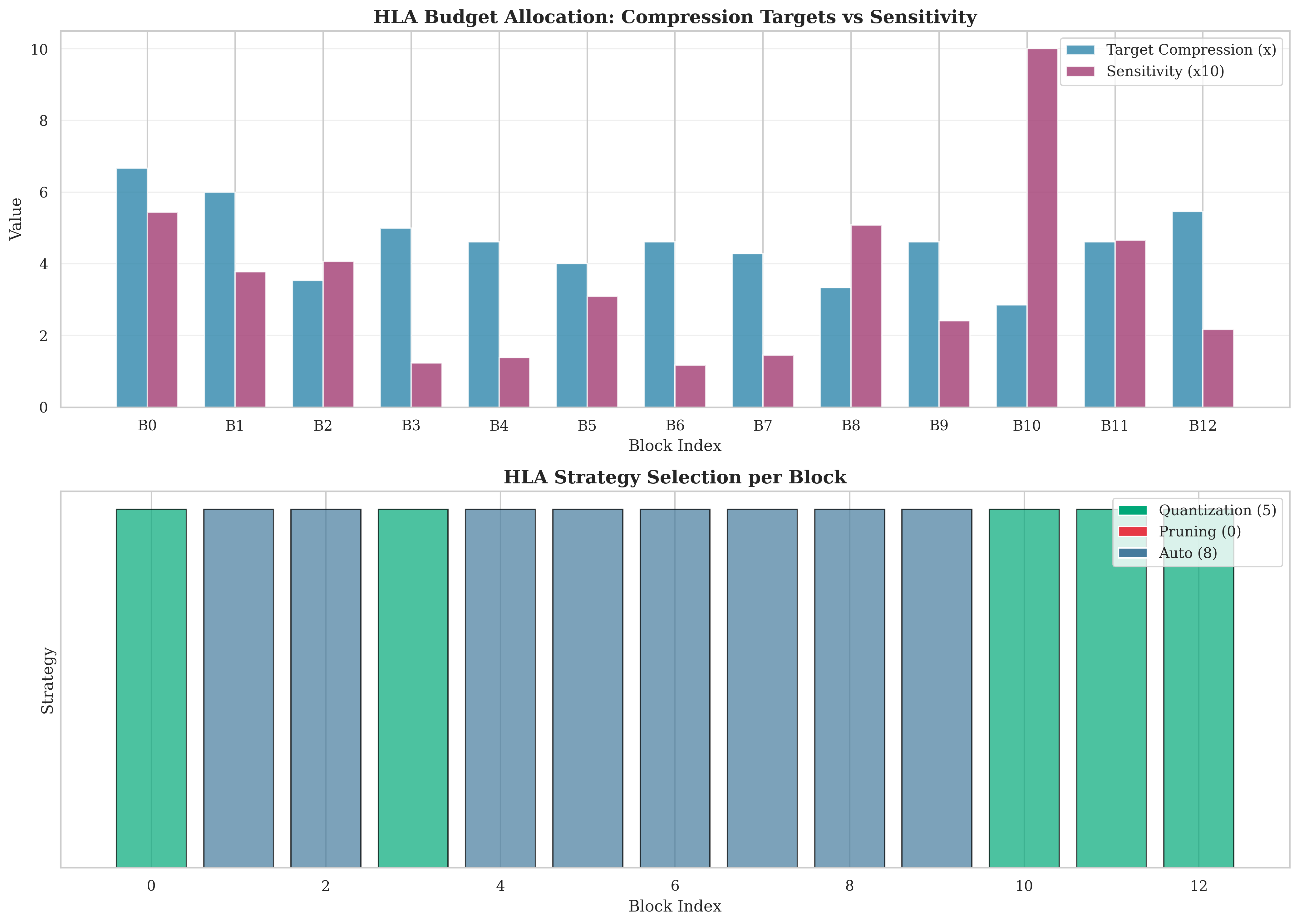}
        \caption{Cycle 1}
    \end{subfigure}
    \par\medskip
    \begin{subfigure}{0.88\linewidth}
        \centering
        \includegraphics[width=\linewidth]{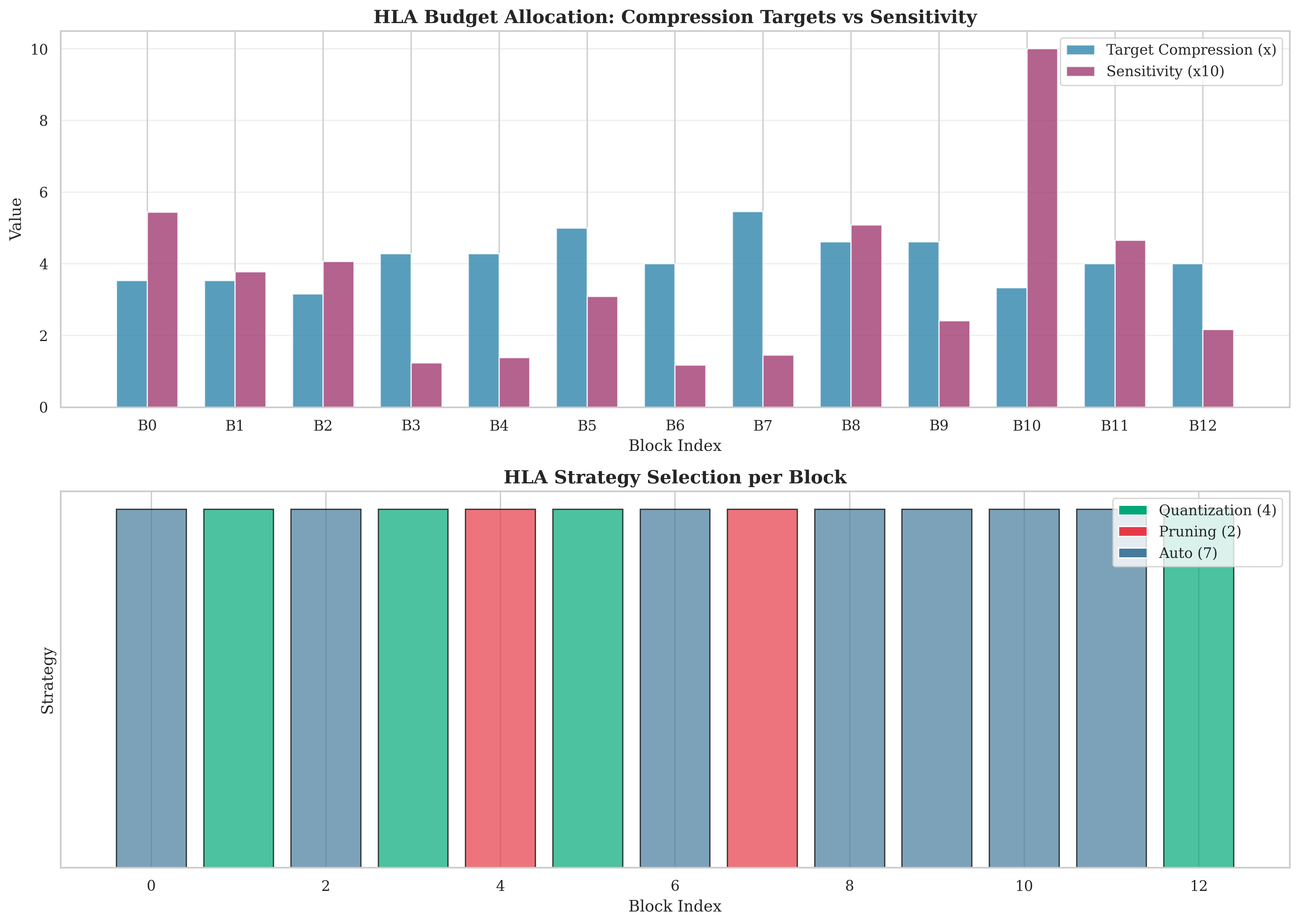}
        \caption{Cycle 2}
    \end{subfigure}
    \caption{MobileNetV2 Tiny ImageNet HLA budget allocation across four HiReLC cycles (cycles 1--2). The vertically stacked plots improve readability of the per-block allocations and show the policy beginning to shift compression pressure toward mid-to-late inverted residual blocks.}
    \label{fig:budget_alloc_all_cycles}
\end{figure}

\begin{figure}[H]
    \ContinuedFloat
    \centering
    \begin{subfigure}{0.88\linewidth}
        \centering
        \includegraphics[width=\linewidth]{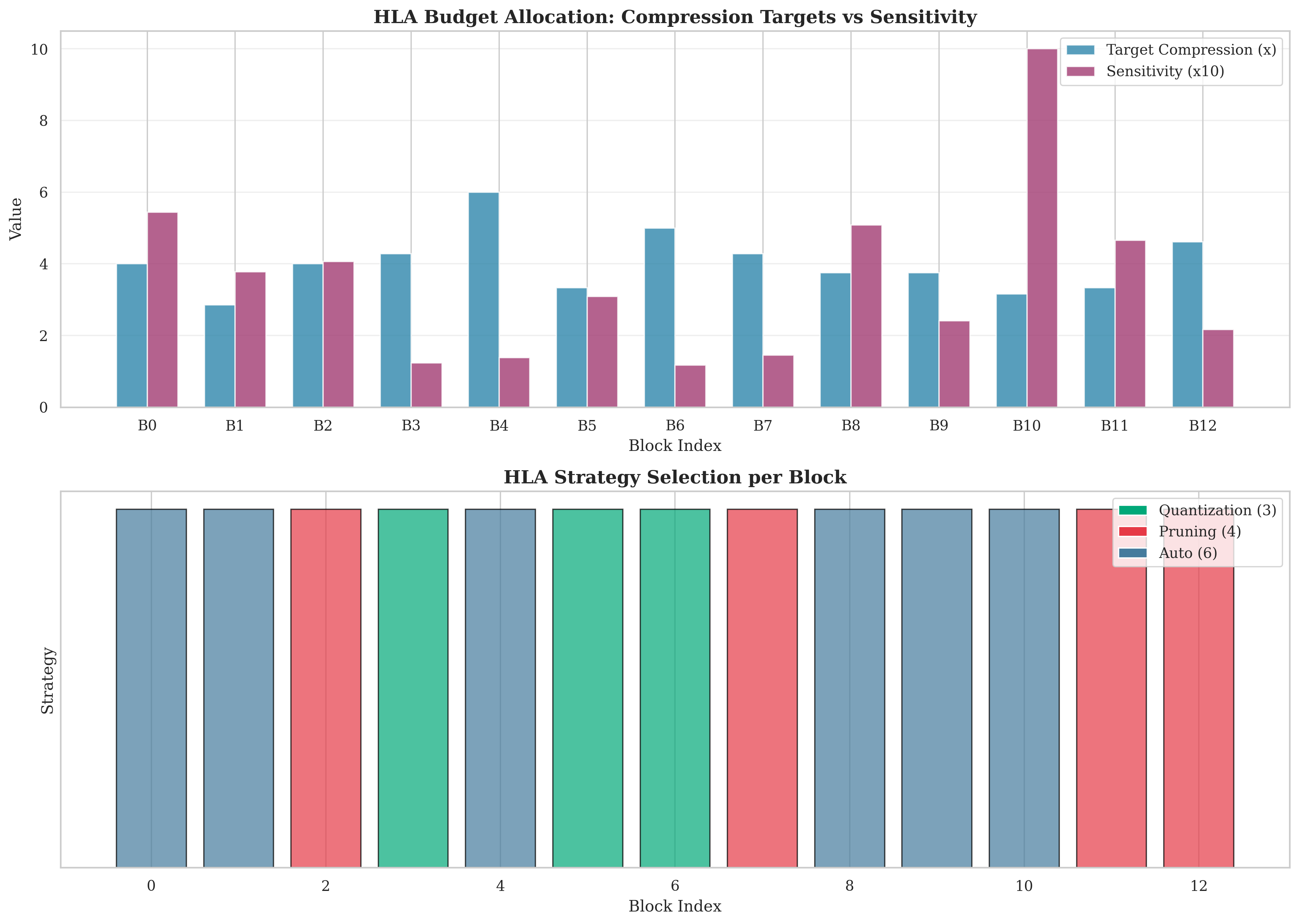}
        \caption{Cycle 3}
    \end{subfigure}
    \par\medskip
    \begin{subfigure}{0.88\linewidth}
        \centering
        \includegraphics[width=\linewidth]{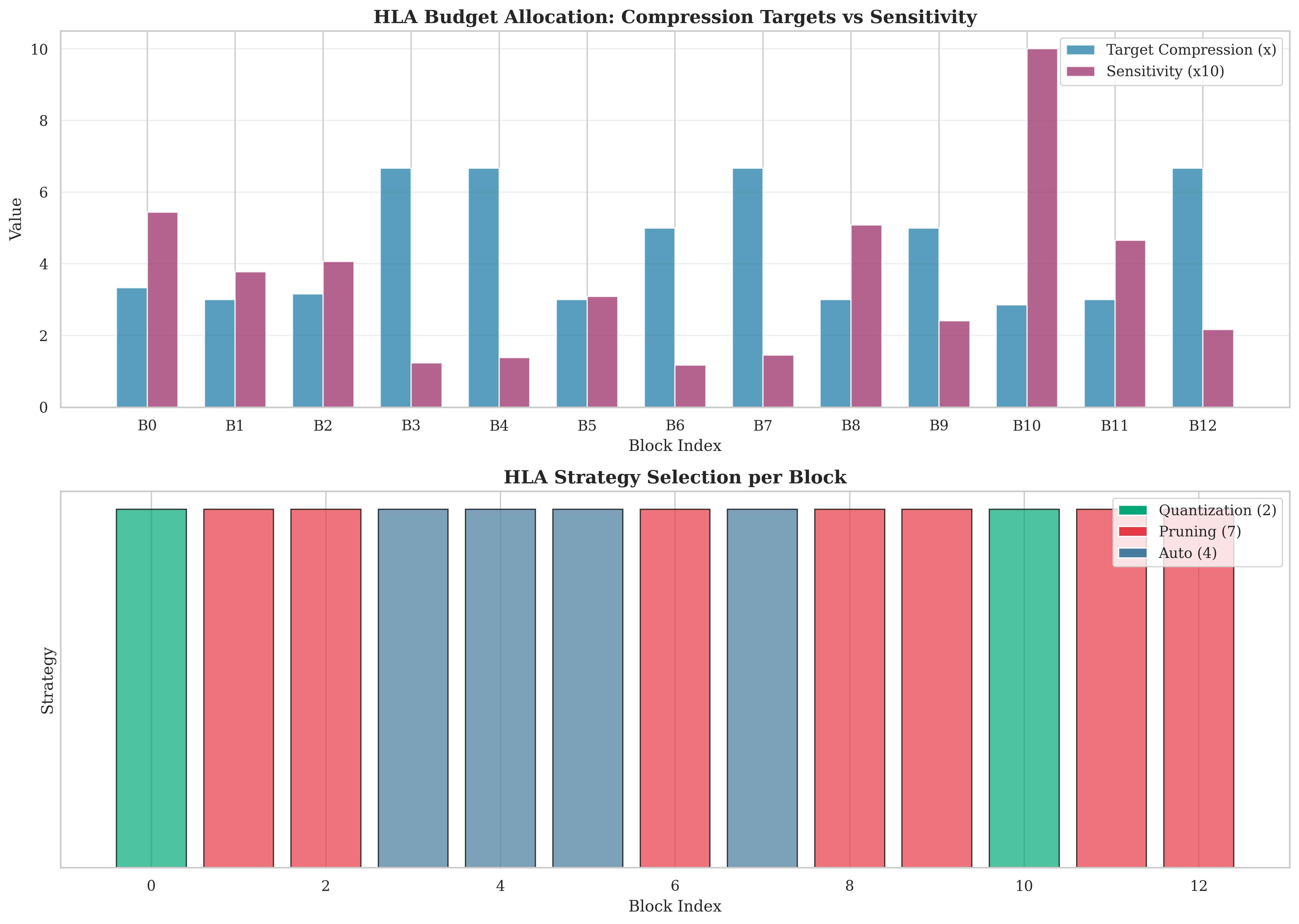}
        \caption{Cycle 4}
    \end{subfigure}
    \caption{MobileNetV2 Tiny ImageNet HLA budget allocation across four HiReLC cycles (cycles 3--4, continued). Later cycles show the policy shifting more strongly toward pruning in mid-to-late blocks.}
\end{figure}

\begin{table}[htbp]
    \caption{Ablation studies on DeiT-Small/CIFAR-100. CR denotes achieved effective parameter-storage compression ratio. The HLA ablation uses uniform budgets with the LLA only; the RL-algorithm ablation compares the default mixed ensemble against a homogeneous A2C ensemble in the first cycle.}
    \label{tab:ablations}
    \begin{center}
        \renewcommand{\arraystretch}{1.25}
        \resizebox{\linewidth}{!}{
        \begin{tabular}{lllll}
            \textbf{Ablation} & \textbf{Configuration} & \textbf{Accuracy Drop} & \textbf{CR} & \textbf{Observation} \\
            \hline \\
            HLA removal & LLA only, uniform budgets & 6.80\% & 8.87$\times$ & Over-compresses and weakens budget control \\
            Full HiReLC & HLA + LLA ensemble & 1.72\% & 6.23$\times$ & Better accuracy retention and controlled CR \\
            RL ensemble & [PPO, A2C, PPO] & 3.28\% & 6.38$\times$ & Balanced first-cycle trade-off \\
            RL ensemble & [A2C, A2C, A2C] & 7.19\% & 6.63$\times$ & More aggressive but less stable \\
        \end{tabular}}
    \end{center}
\end{table}

\begin{figure}[H]
    \centering
    \includegraphics[width=0.95\linewidth]{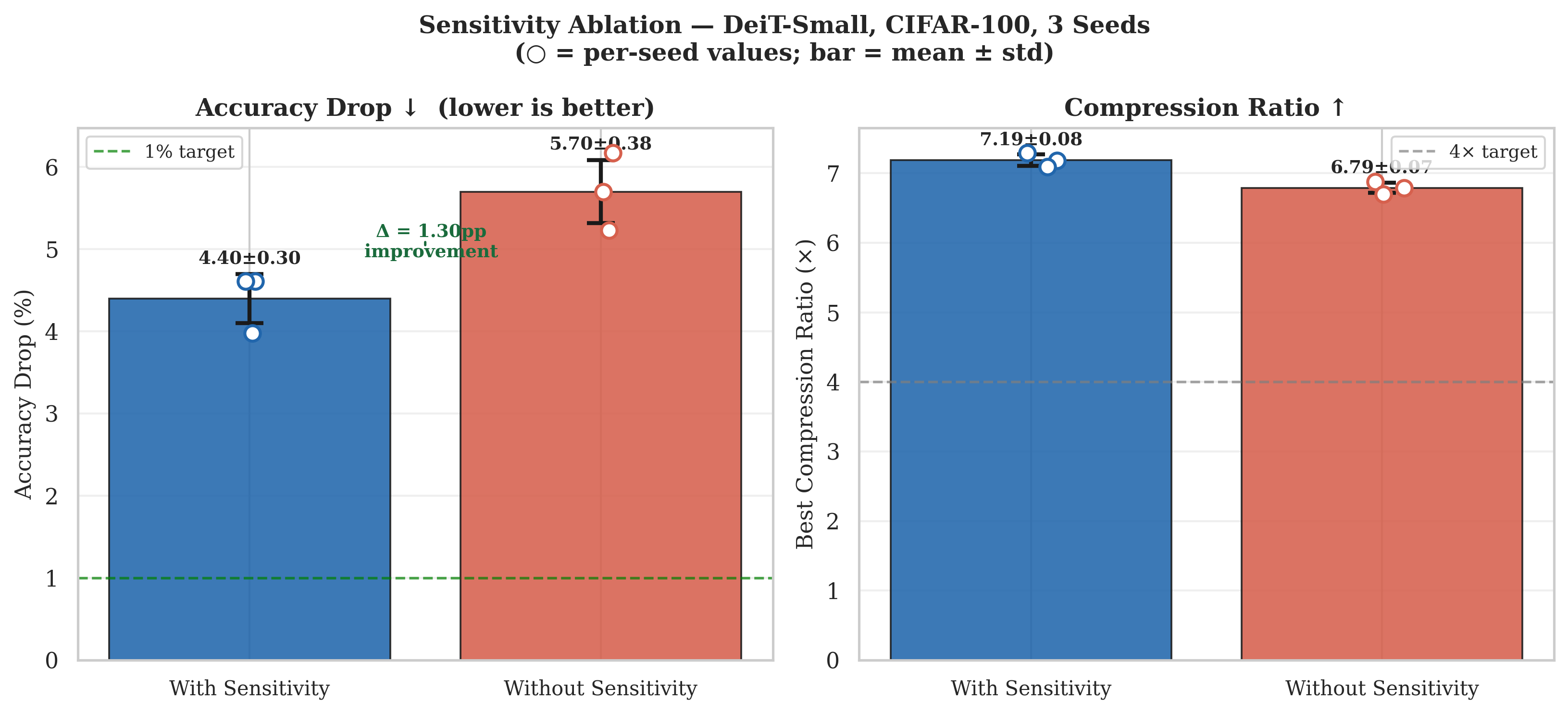}
    \caption{Multi-seed sensitivity ablation on DeiT-Small/CIFAR-100. Fisher-guided sensitivity improves accuracy retention by 1.30 percentage points over the non-sensitivity setting and reduces run-to-run variability.}
    \label{fig:sensitivity_multiseed}
\end{figure}

\begin{figure}[H]
    \centering
    \includegraphics[width=0.75\linewidth]{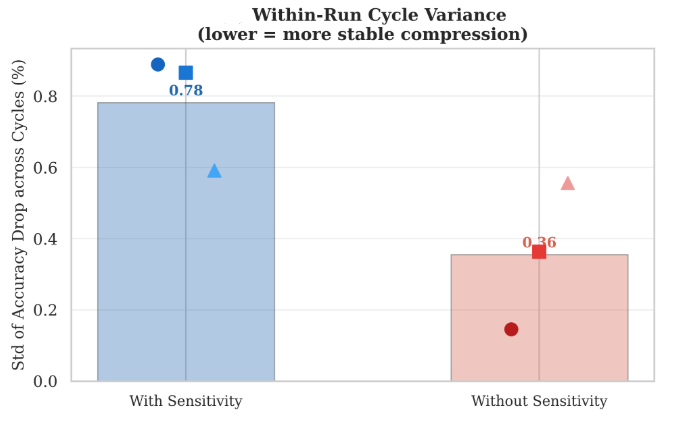}
    \caption{Allocation heterogeneity induced by Fisher sensitivity. Sensitivity guidance increases compression-ratio dispersion across blocks (0.000301 vs. 0.000116), producing more differentiated block-wise budget assignments.}
    \label{fig:sensitivity_variance}
\end{figure}

\begin{figure}[H]
    \centering
    \includegraphics[width=0.95\linewidth]{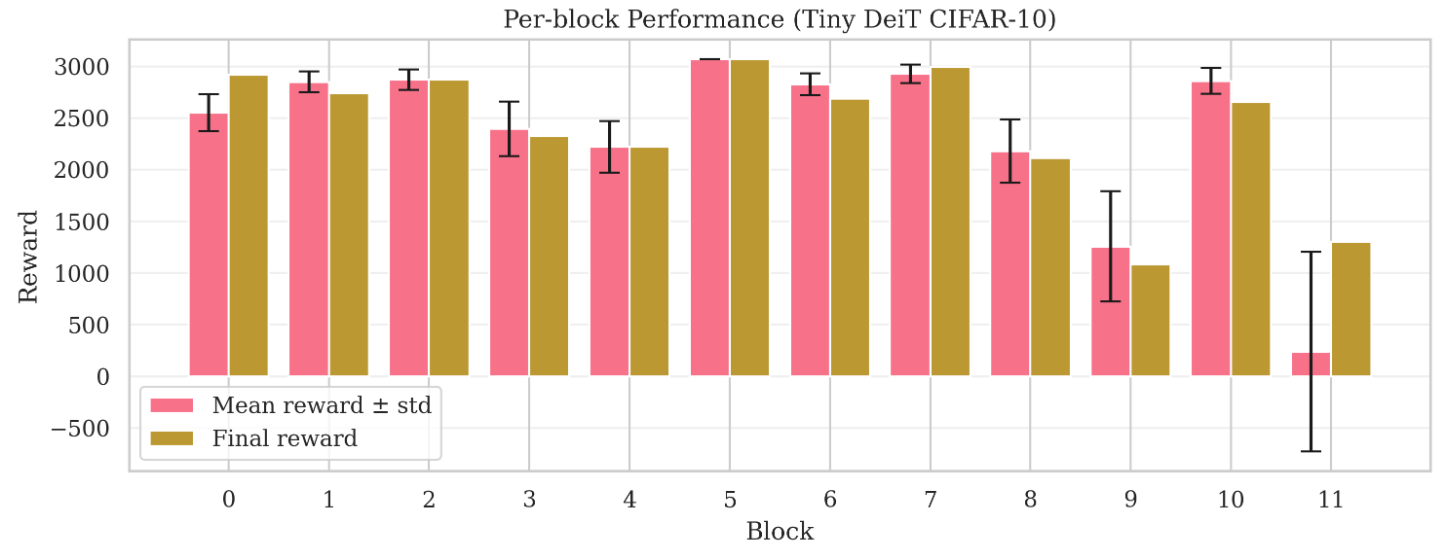}
    \caption{Representative LLA training return curve for DeiT-Small. Mean episode reward improves during the controller update phase and then stabilizes, indicating that the ensemble learns higher-reward compression configurations before cycle-level fine-tuning.}
    \label{fig:lla_reward_curve}
\end{figure}

\section{Contextual Comparison with Prior Compression Methods}
\label{sec:sota}

Table~\ref{tab:sota} gives a contextual comparison with representative compression methods using the metrics reported in their original publications. It is not a same-dataset leaderboard: most prior work reports ImageNet-1K results, while our main HiReLC experiments use CIFAR-10/100 or Tiny ImageNet, and the reported ``compression ratio'' may denote FLOPs, latency, storage, or combined pruning+quantization depending on the method. The comparison therefore highlights scope and trade-offs rather than claiming direct numerical dominance.

\begin{table}[htbp]
    \caption{Contextual comparison with representative compression methods. $\star$ indicates numbers reported on ImageNet-1K in the cited paper; HiReLC rows are our protocol-dependent results on CIFAR-10/100 or Tiny ImageNet. \textbf{CR / Speedup} is not normalized across papers: it denotes FLOPs, latency, storage, or combined pruning+quantization depending on the method. Accuracy $\Delta$ is absolute top-1 change from each method's own baseline, with $(+)$ denoting a drop and $(-)$ denoting a gain.}
    \label{tab:sota}
    \begin{center}
        \renewcommand{\arraystretch}{1.25}
        \resizebox{\linewidth}{!}{
          \begin{tabular}{llllllll}
              \toprule
              \textbf{Method} & \textbf{Model} & \textbf{Dataset} & \textbf{Joint P+Q} & \textbf{CR / Speedup} & \textbf{Acc.\ $\Delta$} & \textbf{Baseline} & \textbf{Search / Protocol} \\
              \midrule
              AMC \citep{he2018amc}
                & MobileNet-V1  & ImageNet$^\star$ & No (prune)
                & $2\times$ FLOPs   & $-0.1\%$    & 70.3\%  & RL (DDPG) \\
              HAQ \citep{wang2019haq}
                & ResNet-50     & ImageNet$^\star$ & No (quant.)
                & $1.4$--$1.95\times$ Latency & $\approx 0\%$ & 76.1\% & RL (DDPG) \\
              HAWQ-V2 \citep{dong2020hawqv2}
                & ResNet-50     & ImageNet$^\star$ & No (quant.)
                & $13\times$ storage & +0.65\% & 76.1\% & Hessian trace \\
              I-ViT \citep{li2023ivit}
                & DeiT-S        & ImageNet$^\star$ & No (quant.)
                & $4\times$ INT8 theoretical & $-0.27\%$ & 79.85\% & QAT (INT8) \\
              DeepCompress-ViT \citep{ahmed2025deepcompress}
                & DeiT-S        & ImageNet$^\star$ & No (storage + quant.)
                & $>14\times$ storage & $<0.5\%$ & 79.8\% & UCT + compres. \\
             & & & & & & &    training \\
              \midrule
              \textbf{HiReLC (ours)} & DeiT-Small    & CIFAR-100     & \textbf{Yes} & \textbf{6.23$\times$ size} & +1.72\% & 86.88\% & \textbf{Hier.\ RL, ours} \\
              \textbf{HiReLC (ours)} & CLIP ViT-B/32 & CIFAR-10      & \textbf{Yes} & \textbf{6.64$\times$ size} & $-3.83\%$ & 89.22\% & \textbf{Hier.\ RL, ours} \\
              \textbf{HiReLC (ours)} & ResNet18      & Tiny ImageNet & \textbf{Yes} & \textbf{5.99$\times$ size} & +4.22\% & 75.16\% & \textbf{Hier.\ RL, ours} \\
              \textbf{HiReLC (ours)} & DeiT-Small    & CIFAR-100     & \textbf{Yes} & \textbf{6.27$\times$ size} & +1.17\% & 87.42\% & \textbf{Hier.\ RL, ours} \\
              \textbf{HiReLC (ours)} & DeiT-Base     & CIFAR-100     & \textbf{Yes} & \textbf{6.72$\times$ size} & +2.11\% & 90.70\% & \textbf{Hier.\ RL, ours} \\
              \bottomrule
          \end{tabular}
        }
    \end{center}
\end{table}

\begin{table*}[t]
\centering
\caption{Contextual ImageNet-1K comparison. Prior-method rows reproduce reported ImageNet-1K numbers from the cited papers. HiReLC rows are our data-efficient ImageNet-1K protocol and should be read as protocol-dependent evidence rather than matched re-implementations. Accuracy $\Delta$ follows the convention: $(+)$ denotes a drop and $(-)$ denotes an improvement relative to the corresponding baseline. CR denotes compression ratio or speedup as reported in each work.}
\label{tab:imagenet_compression}
\scriptsize
\setlength{\tabcolsep}{2.6pt}
\begin{tabular}{@{}lcccccc@{}}
\toprule
\textbf{Method} & \textbf{Year} & \textbf{Base} & \textbf{Acc. $\Delta$ (\%)} & \textbf{CR / Speedup} & \textbf{Metric} & \textbf{Protocol} \\
\midrule

\multicolumn{7}{l}{\textit{ResNet-50 / ImageNet-1K}} \\
AMC \cite{he2018amc} & 2018 & 76.15 & +1.81 & 2.00$\times$ & FLOPs pruning & Original \\
HAQ \cite{wang2019haq} & 2019 & 76.15 & +0.81 & 1.4--1.95$\times$ & latency quant. & Original \\
HAWQ-V2 \cite{dong2020hawqv2} & 2020 & 76.13 & +0.37 & 3.70$\times$ & W2--8 size & Original \\
\textbf{HiReLC (ours)$^{\dagger\ddagger}$} & \textbf{2026} & \textbf{74.12} & \textbf{$\approx$0.00} & \textbf{6.32$\times$} & \textbf{W3--8 + prune size} & \textbf{Subset eval.} \\
\midrule

\multicolumn{7}{l}{\textit{DeiT-Small / ImageNet-1K}} \\
I-ViT \cite{li2023ivit} & 2023 & 79.85 &
\begin{tabular}[c]{@{}c@{}}
+0.68 (PTQ) \\
$-$0.27 (QAT)
\end{tabular}
& 4.0$\times$ & INT8 quant. & Original \\
\textbf{HiReLC (ours)$^{\dagger\ddagger}$} & \textbf{2026} & \textbf{77.18} & \textbf{+2.17} & \textbf{5.59$\times$} & \textbf{W3--8 + prune size} & \textbf{Subset eval.} \\
\bottomrule
\end{tabular}

\vspace{2mm}

\footnotesize
$^\dagger$ HiReLC uses a data-efficient subset of ImageNet-1K
(ResNet-50: 50/1024 shards $\approx$62k images; DeiT-Small:
100/1024 shards $\approx$125k images), evaluated on $\approx$39k validation images.

$^\ddagger$ RL budget is 512 timesteps per LLA agent
(3 agents $\times$ 12--16 blocks). HiReLC accuracy deltas are measured relative to
uncompressed baselines under the same subset protocol; prior baselines are reported using
values from their original publications without re-implementation.
\end{table*}

\begin{figure}[H]
    \centering
    \includegraphics[width=1\linewidth]{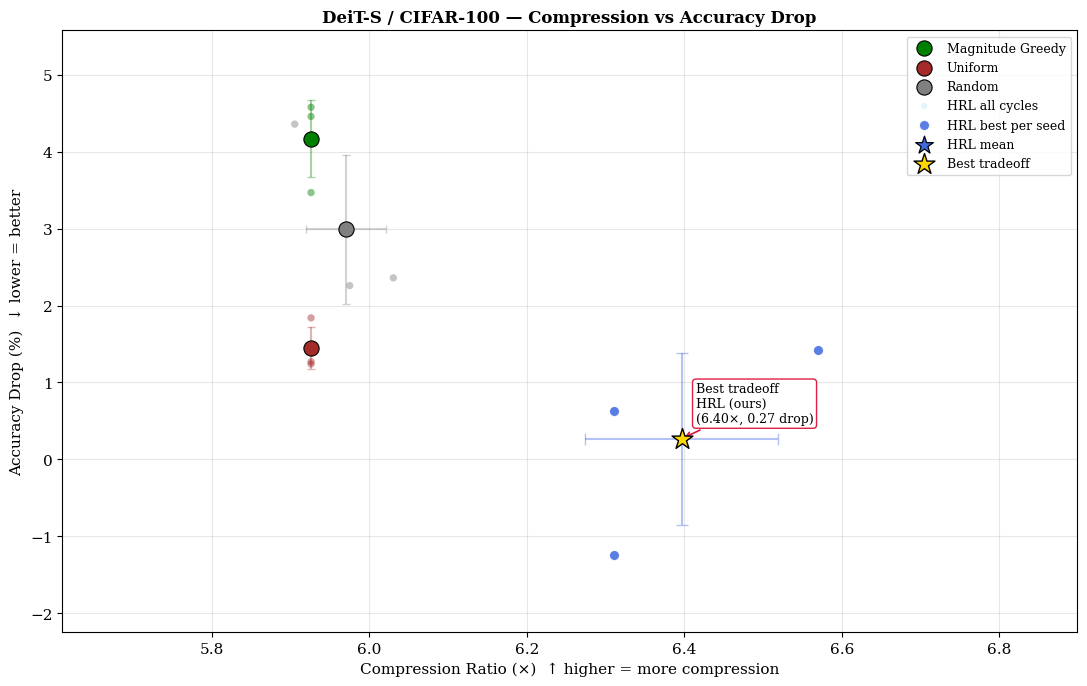}
    \caption{Contextual Pareto view of reported compression ratio versus accuracy drop for HiReLC and representative baselines. Points use each paper's own metric and evaluation protocol, so the figure summarizes trade-off regions rather than a matched benchmark.}
    \label{fig:pareto_baselines}
\end{figure}

Contextualizing the comparison: AMC demonstrates that RL can find effective pruning policies for MobileNet-V1, HAQ uses RL for hardware-aware mixed-precision quantization, HAWQ-V2 shows that Hessian trace information can guide aggressive low-bit quantization, and I-ViT provides strong integer-only ViT inference through QAT. These methods are strong within their targeted compression axis, but they do not jointly search structured pruning and quantization. DeepCompress-ViT is closest in motivation for ViTs, but its reported gains emphasize storage-oriented weight coding and quantization rather than a structured channel-selection policy.

HiReLC should therefore be interpreted as complementary rather than uniformly superior: it achieves effective parameter-storage CRs of 5.99--6.72$\times$ using one hierarchical RL controller across four distinct model families, with 0.55--5.62\% drops in most settings and one CIFAR-10 accuracy gain. The main advantage is the unified search space and controller-level architecture agnosticism; the main caveat is that the absolute numbers are not directly comparable to ImageNet-1K, latency-optimized, or storage-coded baselines without a matched benchmark.

When contextual ImageNet-1K comparisons are included, HiReLC uses a data-efficient
fine-tuning protocol: ResNet-50 uses 50/1024 training shards (approximately 62k images),
DeiT-Small uses 100/1024 shards (approximately 125k images), and evaluation is performed
on a fixed validation subset of approximately 39k images. Reported HiReLC drops are
always relative to the uncompressed baseline under the same protocol; external baselines
are reproduced from their original publications when full retraining is infeasible.

\section{Discussion}
\label{sec:discussion}

\textbf{Framework effectiveness.} HiReLC achieves effective parameter-storage CRs of 5.99--6.72$\times$ with MSRs of 83.3--85.1\% across all architectures, with accuracy drops of 0.55--5.62\% (excluding the CIFAR-10 gain). The per-cycle improvement in Experiment 1 (3.52\%$\to$1.72\%) indicates that the active loop can improve configurations across cycles, with a mild cycle-4 regression reflecting RL stochasticity and the well-documented redundancy in transformer attention heads \citep{michel2019heads,voita2019attention}.

\textbf{CLIP compressibility.} CLIP's contrastive pretraining on 400M image--text pairs \citep{radford2021learning} produces representations with substantial downstream redundancy: Experiment 2 achieves the lowest accuracy drop (0.55\%) at 6.63$\times$ CR. Experiment 3's 3.83\% accuracy \emph{gain} further shows that compression-plus-fine-tuning can act as a beneficial regularizer for over-parameterized vision-language models on simple tasks.

\textbf{CNNs and surrogate accuracy.} CNNs incur larger drops (4.22\%, 5.62\%) than ViTs despite comparable CRs, reflecting their smaller absolute capacity and the harder Tiny ImageNet task. The surrogate is active from cycle 1 in CNN experiments: MAEs of 15.20\% (ResNet18) and 10.37\% (MobileNetV2) are non-negligible, so these values should be viewed as coarse reward-shaping signals rather than calibrated accuracy predictors. The lower MAE for MobileNetV2 is consistent with its more uniform layer structure and smoother optimization landscape. Larger warm-up pools (200--500 samples) would improve surrogate fidelity in future work.

\textbf{Logit-MSE proxy.} ViT experiments use the logit-MSE proxy for three of four cycles with competitive results, suggesting it is a useful low-cost quality signal for transformer-scale compression. Its usefulness is consistent with the smoother representation spaces of ViTs, where output logit fidelity can track downstream accuracy well enough for reward shaping.

\textbf{Budget compliance and HLA dynamics.} Budget compliance of 59.2--62.4\% across experiments indicates consistent over-compression by the LLA, especially early in training, which is expected behavior for broadly exploring RL agents. Therefore, the reported CRs should be interpreted as achieved effective size reductions rather than exact target-compliant outcomes. Future work can tighten compliance by penalizing target deviation more strongly or adding a deterministic post-processing constraint.

\textbf{Sensitivity guidance.} The multi-seed sensitivity ablation in Figures~\ref{fig:sensitivity_multiseed}--\ref{fig:sensitivity_variance} shows that Fisher guidance improves accuracy by 1.30 percentage points over a non-sensitivity variant and produces more structured compression allocations. The guided policy shifts compression toward later blocks (B8--11: 8.25$\times$ vs. 6.66$\times$ without sensitivity) while preserving mid-depth blocks (B2--7: 6.46$\times$ vs. 5.97$\times$), matching the intended role of $s_i$ as a block-importance prior.

\textbf{Effect of quantization configuration (Experiments 6--7).} The mixed INT/FLOAT + log-scale configuration in Experiment 6 achieves a lower drop (1.17\%) than the INT uniform Experiment 1 (1.72\%) at a comparable CR, suggesting per-kernel type flexibility aids accuracy recovery. The DeiT-Base result in Experiment 7 (2.11\% drop at 6.72$\times$ CR, 85.1\% MSR) is the highest-compression result in our study and suggests that larger ViTs may tolerate aggressive joint compression under this configuration. Whether the improvement in Experiment 6 is attributable primarily to the mixed type or log granularity remains to be disentangled in future controlled ablations.

\textbf{Extensibility and Design Choices.} HiReLC is a modular framework rather than a fixed algorithm, and each component can be
replaced or expanded without changing the overall structure. In this work, we report a
specific instantiation (PPO/A2C LLAs, Fisher-based sensitivity, logit-MSE or MLP
surrogate, and a particular goal/budget formulation). These choices can be swapped or
augmented: e.g., alternative sensitivity metrics (Fisher, $L_1$, SNIP, GraSP, Hessian
trace), additional agent types or critics, or goal functions prioritizing latency,
energy, or memory rather than accuracy alone.

Current results also reflect practical training budgets. More cycles, larger warm-up
pools, or stricter budget-penalty terms can be applied within the same framework to
improve compliance and accuracy. Likewise, the surrogate can be strengthened (e.g., with
larger pools or better architectures) without altering the core hierarchical RL design.

\textbf{Limitations.} HiReLC prioritizes effective parameter-storage reduction and does not yet optimize a
measured latency or energy objective, so compression ratios do not necessarily translate
linearly to wall-clock speedups on all hardware. Reported CRs assume packed structured
storage or an equivalent channel-materialization pass and exclude backend-specific sparse
indexing overheads. The surrogate can be noisy for compact
CNNs, as seen in the ResNet18 MAE, and the logit-MSE proxy is only an indirect accuracy
signal. The current budget-compliance term still permits over-compression in early
cycles, and the stability reward is implemented as a compact history-dependent shaping
term rather than by explicitly augmenting the state with the full reward window. Finally,
the comparison tables are contextual rather than matched leaderboards because some baseline
numbers rely on originally reported ImageNet-1K results and full matched retraining of
every baseline is computationally prohibitive.

\textbf{Broader implications.} HiReLC demonstrates that a single unified hierarchical controller can compress diverse architectures, including standard ViTs, vision-language models, and compact CNNs, without architecture-specific controller redesign. The key enabling design choices are: (1) separating global budget allocation (HLA) from per-kernel optimization (LLA); (2) ensemble LLA diversity via PPO+A2C; (3) logit-MSE proxy as a low-cost quality signal; and (4) an optional MLP surrogate to amortize expensive evaluations. Achieved effective parameter-storage CRs of $\approx6\times$ with 0.5--6\% accuracy drops across all tested architectures suggest HiReLC as a promising starting point for deployment settings where model size is the primary constraint; hardware-specific latency or energy optimization remains future work.

\section{Conclusion}
We introduced HiReLC, a hierarchical ensemble-RL framework that jointly optimizes
structured pruning and mixed-precision quantization through sensitivity-aware guidance
and an active learning compression loop. Across diverse architecture families spanning
Vision Transformers and CNNs, the method achieves effective parameter-storage compression
ratios of 5.99--6.72$\times$ with accuracy changes ranging from a 3.83\%
gain to a 5.62\% drop, while using ensemble agents with heterogeneous reward weights to
stabilize the discrete compression search.

Future work will explore hardware-in-the-loop objectives, broader backbones, matched
re-implementations of representative baselines, and tighter integration of latency and
energy models into the reward. The existing framework can also support stricter
constraint handling by adding a deterministic post-processing step that relaxes
compression on the most sensitive blocks when a cycle exceeds the target accuracy drop.

\bibliographystyle{plainnat} 
\bibliography{main}         

\newpage


\appendix

\section*{A - Proofs and Derivations}

\subsection*{A.1 \; Reward Component Scale Compatibility}

We verify that the five components of the LLA reward (Eq.~\ref{eq:lla_reward}) operate
on compatible empirical scales, so that no single term dominates the signal by
construction.

\subsubsection*{A.1.1. Accuracy reward $R_{\mathrm{acc}} \in [0, 100]$}
$R_{\mathrm{acc}}$ is defined as a piecewise-linear function of the estimated accuracy
drop $\Delta \hat{\mathcal{A}} \geq 0$:
\begin{equation}
    R_{\mathrm{acc}}(\Delta\hat{\mathcal{A}}) =
    \begin{cases}
        100 & \Delta\hat{\mathcal{A}} < 1 \\
        95 - 5(\Delta\hat{\mathcal{A}} - 1)   & 1 \leq \Delta\hat{\mathcal{A}} < 2 \\
        85 - 10(\Delta\hat{\mathcal{A}} - 2)  & 2 \leq \Delta\hat{\mathcal{A}} < 3 \\
        \max\bigl(0,\; 70 - 10(\Delta\hat{\mathcal{A}} - 3)\bigr) & \Delta\hat{\mathcal{A}} \geq 3
    \end{cases}
\end{equation}

Since $\Delta\hat{\mathcal{A}} \geq 0$ and the function is non-increasing with a floor
at $0$, we have $R_{\mathrm{acc}} \in [0, 100]$ for all inputs. The function is
piecewise-linear and non-increasing: at $\Delta\hat{\mathcal{A}} = 1$, the reward
changes from $100$ just below the breakpoint to $95$ at the breakpoint, and larger
drops are penalized more severely by the subsequent slopes ($-5$, $-10$, $-10$),
matching the asymmetric cost of accuracy degradation in compression tasks.

\subsubsection*{A.1.2. Compression reward $R_{\mathrm{comp}} \in [0, 60)$}
From Eq.~\eqref{eq:compression_ratio}, $\nu_{i,k} = (b_{i,k}/32)\cdot\rho_{i,k}$,
so $\bar{\nu}_i \in (0, 1]$ since $b_{i,k} \leq 32$ and $\rho_{i,k} \in (0,1]$.
Therefore:
\begin{equation}
    R_{\mathrm{comp}} = 60 \cdot (1 - \bar{\nu}_i) \in [0,\; 60)
\end{equation}

The upper bound of $60$ is approached as $\bar{\nu}_i \to 0$ (extreme compression);
$R_{\mathrm{comp}} = 0$ when $\bar{\nu}_i = 1$ (no compression). The maximum is
strictly less than $60$ since $\rho_{i,k} > 0$ by definition.

\subsubsection*{A.1.3. Budget compliance reward $R_{\mathrm{budget}} \in (0, 20]$}
Let $\epsilon_i = |\bar{\nu}_i - R_i| \geq 0$ be the deviation from the HLA-assigned
target. Then:
\begin{equation}
    R_{\mathrm{budget}} = 20 \cdot \exp(-10\,\epsilon_i)
\end{equation}

Since $\exp(-10\,\epsilon_i) \in (0, 1]$ for $\epsilon_i \geq 0$, we have
$R_{\mathrm{budget}} \in (0, 20]$. The maximum of $20$ is attained exactly when
$\bar{\nu}_i = R_i$ (perfect compliance). The reward decays exponentially: at
$\epsilon_i = 0.1$, $R_{\mathrm{budget}} \approx 7.36$; at $\epsilon_i = 0.3$,
$R_{\mathrm{budget}} \approx 1.00$, providing a sharp but smooth incentive for budget
adherence.

\subsubsection*{A.1.4. Sensitivity penalty $P_{\mathrm{sens}} \in [0, 12)$}
Proved in full in Section~A.2 below.

\subsubsection*{A.1.5. Stability component $R_{\mathrm{stab}} \leq 0$}
$R_{\mathrm{stab}} = -2 \cdot \mathrm{std}(\mathcal{H}_{t-1})$, where $\mathcal{H}_{t-1}$
is the previous LLA reward window and the term is set to $0$ when the history is empty.
Since standard deviation is non-negative, $R_{\mathrm{stab}} \leq 0$ always, with equality
when all rewards in the window are identical. It acts as a pure regularizer.

\paragraph{Scale compatibility summary}
The five components span ranges $[0,100]$, $[0,60)$, $(0,20]$, $\leq 0$, and $[0,12)$.
The ensemble weights $(\alpha_j, \beta_j, \gamma_j)$ further modulate the first three
terms and the stability component. No single term is structurally dominant: the maximum
attainable total reward is bounded by
$\alpha_j \cdot 100 + \beta_j \cdot 60 + 20$, and the maximum penalty is
$\gamma_j \cdot |R_{\mathrm{stab}}| + 12$, ensuring meaningful gradient signal
across all components throughout training.

\subsection*{A.2 \; Bounds and Monotonicity of the Sensitivity Penalty}

\begin{proposition}
$P_{\mathrm{sens}} = 12 \cdot s_i \cdot (1 - \bar{\nu}_i) \in [0, 12)$ for all valid
compression configurations, and is jointly monotone increasing in $s_i$ and in
$(1-\bar{\nu}_i)$.
\end{proposition}

\begin{proof}
By definition (Eq.~\eqref{eq:sensitivity_score}), $s_i \in [0,1]$ as it is
max-normalized:\\ 
$$s_i = \sum_{\theta_j \in B_i}\hat{F}(\theta_j) \,/\,
\max_{i'}\sum_{\theta_j \in B_{i'}}\hat{F}(\theta_j),$$
 where $\hat{F}(\theta_j) \geq 0$
(squared gradients), so $s_i \geq 0$, and the numerator is at most the denominator,
so $s_i \leq 1$.

From Eq.~\eqref{eq:compression_ratio}, $\nu_{i,k} = (b_{i,k}/b_0)\cdot\rho_{i,k}$
with $b_{i,k} \leq b_0 = 32$ and $\rho_{i,k} \in (0,1]$, giving $\nu_{i,k} \in (0,1]$.
Since $\bar{\nu}_i$ is an average of such values, $\bar{\nu}_i \in (0,1]$, hence
$(1-\bar{\nu}_i) \in [0,1)$.
Therefore:
\begin{equation}
    P_{\mathrm{sens}} = 12 \cdot \underbrace{s_i}_{\in[0,1]} \cdot
    \underbrace{(1-\bar{\nu}_i)}_{\in[0,1)} \;\in\; [0,\;12)
    \;\subset\; [0,12]
\end{equation}

The upper bound of $12$ is approached as $s_i \to 1$ and $\bar{\nu}_i \to 0$
(maximally sensitive block under maximal compression), but is never attained since
$\bar{\nu}_i > 0$ strictly.

Monotonicity follows immediately from the product form: holding $\bar{\nu}_i$ fixed,
$\partial P_{\mathrm{sens}}/\partial s_i = 12(1-\bar{\nu}_i) \geq 0$; holding $s_i$
fixed, $\partial P_{\mathrm{sens}}/\partial(1-\bar{\nu}_i) = 12 s_i \geq 0$.
Hence $P_{\mathrm{sens}}$ penalizes aggressively compressing sensitive blocks
(high $s_i$, low $\bar{\nu}_i$) and is zero whenever the block is uncompressed
($\bar{\nu}_i = 1$) or sensitivity is zero ($s_i = 0$).
\end{proof}

\noindent\textbf{Remark.} The design is intentional: a block with $s_i = 1.0$
compressed to $\bar{\nu}_i = 0.1$ incurs $P_{\mathrm{sens}} = 12 \times 0.9 = 10.8$,
nearly nullifying $R_{\mathrm{budget}}$'s maximum of $20$ and strongly discouraging
the policy from aggressively compressing critical blocks without a commensurate accuracy
recovery signal.

\subsection*{A.3 \; Compression Ratio: Derivation and Theoretical Bounds}

\paragraph{A.3.1. Derivation of the global CR from per-kernel decisions}
The storage cost of a weight tensor under mixed-precision compression is proportional
to the number of bits required. For kernel $k$ in block $B_i$, with $C_{\mathrm{out}}$
output channels and $P_{i,k}$ total parameters, the compressed storage relative to
FP32 is:
\begin{equation}
    \nu_{i,k} \;=\; \frac{b_{i,k}}{32} \cdot \rho_{i,k}
    \label{eq:cr_derivation}
\end{equation}
where $b_{i,k}/32$ captures the bitwidth reduction and $\rho_{i,k}$ captures the
fraction of channels retained after structured pruning. The accounting assumes a packed
structured representation, or an equivalent materialization pass that removes masked
channels; it intentionally does not claim hardware latency without backend-specific
compilation. Taking the block-level parameter-weighted average over the $K_i$ kernels
in block $B_i$:
\begin{equation}
    \bar{\nu}_i = \frac{1}{P_i} \sum_{k=1}^{K_i} P_{i,k}\nu_{i,k}
    = \frac{1}{P_i} \sum_{k=1}^{K_i} P_{i,k}\frac{b_{i,k}}{32} \cdot \rho_{i,k},
\end{equation}
and the global retained-size fraction, again weighted by FP32 parameter counts:
\begin{equation}
    \bar{\nu} = \frac{1}{P} \sum_{i=1}^{N} P_i\bar{\nu}_i
    = \frac{1}{P} \sum_{i=1}^{N} \sum_{k=1}^{K_i} P_{i,k}
      \frac{b_{i,k}}{32} \cdot \rho_{i,k}
\end{equation}
The global compression ratio is then $\mathrm{CR} = 1/\bar{\nu}$
(Eq.~\eqref{eq:global_cr}).

\paragraph{A.3.2. Theoretical bounds given configured ranges}
Let $b_{\min}, b_{\max}$ and $\rho_{\min}^{\mathrm{prune}}, \rho_{\max}^{\mathrm{prune}}$
be the globally configured bitwidth and pruning bounds. The keep-ratio satisfies
$\rho_{i,k} \in [1 - \rho_{\max}^{\mathrm{prune}},\; 1 - \rho_{\min}^{\mathrm{prune}}]$
(Eq.~\eqref{eq:prune_decode}). Then the per-kernel retained-size fraction is bounded by:
\begin{equation}
    \nu_{i,k}^{\min} = \frac{b_{\min}}{32} \cdot (1 - \rho_{\max}^{\mathrm{prune}}),
    \qquad
    \nu_{i,k}^{\max} = \frac{b_{\max}}{32} \cdot (1 - \rho_{\min}^{\mathrm{prune}})
\end{equation}
Since $\bar{\nu}$ is an average of $\nu_{i,k}$ values, it inherits these bounds:
$\bar{\nu} \in [\nu^{\min}, \nu^{\max}]$, giving achievable CR range:
\begin{equation}
    \mathrm{CR} \;\in\;
    \left[\frac{1}{\nu^{\max}},\;\frac{1}{\nu^{\min}}\right]
    \;=\;
    \left[
        \frac{32}{b_{\max}(1-\rho_{\min}^{\mathrm{prune}})},\;
        \frac{32}{b_{\min}(1-\rho_{\max}^{\mathrm{prune}})}
    \right]
\end{equation}

\noindent\textbf{Instantiation for our experiments.} With
$b_{\min}=4$, $b_{\max}=8$, $\rho_{\min}^{\mathrm{prune}}=0$,
$\rho_{\max}^{\mathrm{prune}}=0.6$:
\begin{equation}
    \nu^{\min} = \frac{4}{32} \cdot 0.4 = 0.05, \qquad
    \nu^{\max} = \frac{8}{32} \cdot 1.0 = 0.25
\end{equation}
\begin{equation}
    \mathrm{CR} \;\in\; \left[\frac{1}{0.25},\; \frac{1}{0.05}\right]
    = [4\times,\; 20\times]
\end{equation}
The empirically achieved effective parameter-storage CRs of $5.99$--$6.72\times$ across
all experiments fall well within this theoretical range, showing that the HLA and LLA
jointly explore a meaningful subspace of the achievable compression frontier rather than
saturating either extreme.

\paragraph{A.3.3. Relationship to the target constraint}
The global constraint $\mathcal{R}(\mathbf{c}) \geq R_{\mathrm{target}}$
(Eq.~\eqref{eq:constraints}) is equivalent to $\bar{\nu} \leq \bar{\nu}^* = 1/R_{\mathrm{target}}$.
For the constraint to be feasible, we require $\bar{\nu}^* \geq \nu^{\min}$, i.e.:
\begin{equation}
    R_{\mathrm{target}} \;\leq\; \frac{32}{b_{\min}(1-\rho_{\max}^{\mathrm{prune}})}
\end{equation}
With our settings this gives $R_{\mathrm{target}} \leq 20$, so the configured target
of $R_{\mathrm{target}} = 4$ ($\bar{\nu}^* = 0.25$) is well within the feasible region.

\section*{B - Hyperparameters Range}
\label{sec:experiments_hyperparams}

\begin{table}[H]
    \caption{Shared hyperparameters across all experiments.}
    \label{tab:hyperparams}
    \begin{center}
        \begin{tabular}{lll}
            \textbf{Hyperparameter} & \textbf{Value} & \textbf{Notes} \\
            \hline \\
            LLA ensemble size        & 3              & 2 PPO + 1 A2C \\
            HLA ensemble size        & 3              & \\
            LLA / HLA training timesteps & 256        & per block / per cycle \\
            Number of cycles ($N_{\text{cycles}}$) & 1--4 & \\
            Max fine-tuning epochs   & 15--20 & patience = 3 \\
            Fine-tuning learning rate & $5 \times 10^{-5}$ & AdamW \\
            Batch size               & 64             & \\
            Quantization bitwidth range & $b \in \{4,\ldots,8\}$ & per kernel \\
            Target CR lower bound & $R_{\mathrm{target}} = 4$ & equivalent retained-size target $\bar{\nu}^*=0.25$ \\
            Surrogate MLP hidden dims & $[64, 32]$    & 6{,}017--11{,}393 params \\
            CNN warm-up pool size    & 50--100 samples    & Experiments 4--5 only \\
        \end{tabular}
    \end{center}
\end{table}

\section*{C - Complete Algorithm}
\label{sec:complete_algorithmic_implementation}

Algorithm~\ref{alg:hirelc} provides a unified formal summary of HiReLC.

\begin{algorithm}[t]
\tiny
\DontPrintSemicolon
\SetAlgoLined
\SetAlgoNlRelativeSize{0}
\SetKwInOut{Input}{Input}
\SetKwInOut{Output}{Output}

\Input{
  Pre-trained $\mathcal{M}(\boldsymbol{\theta})$;\quad
  $\mathcal{G} = (\Delta_{\mathrm{acc}},\, R_{\mathrm{target}},\,
   b_{\min},\, b_{\max},\, \rho_{\min}^{\mathrm{prune}},\, \rho_{\max}^{\mathrm{prune}},\, N_{\mathrm{rep}}^{\min})$;\quad
  \texttt{quantization\_type} $\in$ \{\texttt{mixed}, \texttt{int}, \texttt{float}\};\quad
  \texttt{strategy} $\in$ \{$\emptyset$, \texttt{uniform}, \texttt{log}, \texttt{per-channel}, \texttt{learned}\};\quad
  cycles $T$
}
\Output{Best post-compression fine-tuned parameters $\boldsymbol{\theta}^*$ and configuration $\mathbf{c}^*$}

\textbf{// Phase 0: Initialisation}

Parse $\mathcal{M}$ into $N$ blocks $\{B_i\}_{i=1}^N$, each with $K_i$ kernels\;
Fine-tune $\rightarrow \boldsymbol{\theta}_0$;\quad
$\mathcal{A}_{\mathrm{base}} \leftarrow \mathrm{Evaluate}(\boldsymbol{\theta}_0)$\;
Compute $\{s_i\}$ via Eqs.~\eqref{eq:fisher}--\eqref{eq:sensitivity_score}\;
Init surrogate MLP $\hat{f}$;\quad
$\mathcal{B}_{\mathrm{rep}} \leftarrow \emptyset$\;

\textbf{// Phase 1: Agent Initialisation}

Init $n_h$ HLA agents $\{\pi_h^{\mathrm{HLA}}\}$ (PPO/A2C) with heterogeneous weights;\quad pre-train on simulated env\;
Init $n_a$ LLA agents $\{\pi_{i,j}^{\mathrm{LLA}}\}$ per block with heterogeneous weights;\quad
$w_j \leftarrow 1/n_a\;\forall j$\;

\textbf{// Phase 2: Active Hierarchical Compression Loop}

$\mathbf{c}^* \leftarrow \emptyset$;\quad $\mathcal{A}^* \leftarrow 0$\;
$\Delta\mathcal{A}^{(0)} \leftarrow 0$;\quad $\bar{\nu}^{(0)} \leftarrow 1$\;

\For{$t \leftarrow 1$ \KwTo $T$}{

  \tcp{HLA: update feedback, optionally re-adapt, allocate budgets}
  Update HLA states with $(\Delta\mathcal{A}^{(t-1)},\; \bar{\nu}^{(t-1)},\; t/T)$\;
  \lIf{$t > 1$}{re-train HLA ensemble for a short burst}

  \For{each HLA agent $h$}{
    \lIf{$t = T$}{$\mathbf{a}_h \leftarrow \arg\max\,\pi_h^{\mathrm{HLA}}(\cdot|\mathbf{s})$ \emph{(exploit)}}
    \lElse{$\mathbf{a}_h \sim \pi_h^{\mathrm{HLA}}(\cdot|\mathbf{s})$ \emph{(explore)}}
  }
  \For{each block $B_i$}{
    $(R_i,\,\rho_{\max,i},\,b_{\min,i},\,p_i,\,\Delta_{\mathrm{acc},i})
      \leftarrow \tfrac{1}{n_h}\sum_h \mathrm{DecodeNumerical}(\mathbf{a}_h, i)$\;
    $\sigma_i \leftarrow \operatorname{mode}\!\left(\{\mathrm{DecodeStrategy}(\mathbf{a}_h, i)\}_h\right)$\;
    \lIf{$s_i > 0.7$}{adjust $(R_i,\rho_{\max,i},b_{\min,i})$ one tier lighter}
    \lElseIf{$s_i < 0.3$}{adjust $(R_i,\rho_{\max,i},b_{\min,i})$ one tier more aggressive}
    $\mathcal{B}_i^{(t)} \leftarrow \mathrm{LayerBudget}(R_i,\rho_{\max,i},b_{\min,i},\sigma_i,s_i,\Delta_{\mathrm{acc},i})$\;
  }

    \tcp{LLA: train agents per block, extract configurations}
  \For{each block $B_i$}{
    Cache $\mathbf{z}_{\mathrm{base}} \leftarrow \mathcal{M}(\mathcal{X}_c;\boldsymbol{\theta}_0)$;\quad
    store $\mathbf{W}_{i,k}^{(0)} \leftarrow \mathbf{W}_{i,k}\;\forall k$\;

    \For{each LLA agent $j$}{
      \For{each env step (max 20)}{
        Sample $\mathbf{a}_t \sim \pi_{i,j}(\cdot|\mathbf{s}_t)$;\quad
        decode candidate block configuration $c_i(\mathbf{a}_t)$ with $(b_{i,k},\rho_{i,k})_{k=1}^{K_i}$\;
        Set $\tau_{i,k}$: forced if \texttt{quantization\_type} $\neq$ \texttt{mixed}; else agent choice\;
        Set $\mu_{i,k}$: forced if \texttt{strategy} set; else agent choice\;
        Apply $\mathcal{C}$: prune (Eq.~\ref{eq:pruning_mask}), then quantize (Eq.~\ref{eq:quantization})\;
        \eIf{$|\mathcal{B}_{\mathrm{rep}}| \geq N_{\mathrm{rep}}^{\min}$}{
          $R_{\mathrm{acc}} \leftarrow \phi(\mathcal{A}_{\mathrm{base}} - \hat{f}(c_i(\mathbf{a}_t)))$
          \hfill\emph{surrogate}
        }{
          $R_{\mathrm{acc}} \leftarrow \phi(\mathrm{MSE}_{\mathrm{logit}})$ \hfill\emph{Eq.~\ref{eq:logit_mse}}
        }
        Compute $r_t^{\mathrm{LLA}}$ (Eq.~\ref{eq:lla_reward});\quad store transition;\quad
        restore $\mathbf{W}_{i,k} \leftarrow \mathbf{W}_{i,k}^{(0)}$\;
      }
      Update $\pi_{i,j}$ via PPO or A2C;\quad clip gradient norms\;
    }

    $c_i^{(t)} \leftarrow
      \operatorname{round}(\sum_j w_j\,\mathbf{a}_{i,j}^*)$;\quad
    apply type/granularity overrides
    $\Rightarrow \{b_{i,k},\rho_{i,k},\tau_{i,k},\mu_{i,k}\}_{k=1}^{K_i}$
    \hfill\emph{(Eq.~\ref{eq:voting})}\;
  }
  $\mathbf{c}^{(t)} \leftarrow \{c_i^{(t)}\}_{i=1}^N$\;

    \tcp{Apply compression, fine-tune, update surrogate, track best}
  $\boldsymbol{\theta}^{(t)} \leftarrow \mathcal{C}(\boldsymbol{\theta}_0,\mathbf{c}^{(t)})$
  \hfill\emph{(Eqs.~\ref{eq:pruning_mask},\,\ref{eq:quantization})}\;
  Fine-tune $\boldsymbol{\theta}^{(t)}$ via AdamW; re-apply masks each step; patience early stopping\;
  $\mathcal{A}^{(t)} \leftarrow \mathrm{Evaluate}(\boldsymbol{\theta}^{(t)})$\;
  $\mathcal{B}_{\mathrm{rep}} \leftarrow \mathcal{B}_{\mathrm{rep}} \cup \{(\mathbf{c}^{(t)}, \mathcal{A}^{(t)})\}$;\quad
  re-train $\hat{f}$ if $|\mathcal{B}_{\mathrm{rep}}| \geq N_{\mathrm{rep}}^{\min}$
  \hfill\emph{(Eq.~\ref{eq:surrogate_loss})}\;
  \If{$\mathcal{A}^{(t)} > \mathcal{A}^*$}{
    $\mathbf{c}^* \leftarrow \mathbf{c}^{(t)}$;\quad
    $\boldsymbol{\theta}^* \leftarrow \boldsymbol{\theta}^{(t)}$;\quad
    $\mathcal{A}^* \leftarrow \mathcal{A}^{(t)}$\;
  }
  $\Delta\mathcal{A}^{(t)} \leftarrow \mathcal{A}_{\mathrm{base}} - \mathcal{A}^{(t)}$;\quad
  $\bar{\nu}^{(t)} \leftarrow P^{-1}\sum_i P_i\bar{\nu}_i^{(t)}$
  \hfill\emph{(Eq.~\ref{eq:global_cr})}\;
}

\Return $(\boldsymbol{\theta}^*,\mathbf{c}^*)$\;

\caption{HiReLC: Hierarchical RL Framework for Neural Network Compression}
\label{alg:hirelc}
\end{algorithm}

\end{document}